%% file: main.tex
\documentclass[10pt]{article}
\usepackage[preprint]{tmlr}
\usepackage{hyperref}
\hypersetup{hidelinks}
\usepackage{url}
\usepackage{amsmath,amssymb,amsfonts,bm}
\usepackage{booktabs}
\usepackage{tabularx}
\usepackage{makecell}
\usepackage{multirow}
\usepackage[section]{placeins}
\usepackage{enumitem}
\setlist{nosep,leftmargin=1.5em}
\usepackage{graphicx}
\usepackage{subcaption}
\usepackage{wrapfig}
\usepackage{algorithm}
\usepackage{algpseudocode}
\usepackage[table]{xcolor}
\definecolor{top1}{HTML}{8CCFC4}\definecolor{top2}{HTML}{BFE3DC}\definecolor{top3}{HTML}{E3F3F0}

\def\eqref#1{equation~\ref{#1}}
\def\ve{{\bm{e}}}
\def\vk{{\bm{k}}}
\def\vm{{\bm{m}}}
\def\vo{{\bm{o}}}
\def\vq{{\bm{q}}}
\def\vu{{\bm{u}}}
\def\vv{{\bm{v}}}
\def\vw{{\bm{w}}}
\def\vy{{\bm{y}}}
\def\vz{{\bm{z}}}
\def\mK{{\bm{K}}}
\def\mM{{\bm{M}}}
\def\mQ{{\bm{Q}}}
\def\mV{{\bm{V}}}
\def\mW{{\bm{W}}}
\def\mX{{\bm{X}}}
\def\mZ{{\bm{Z}}}
\newcommand{\R}{\mathbb{R}}
\newcommand{\openrfm}{OpenRFM}

\title{OpenRFM: Democratizing Frontier Predictive Intelligence on Relational Data}

\author{\name Zhikai Chen$^{1*}$\quad Junyu Yin$^{4*}$\quad Jialiang Gu$^{4*}$\quad Siheng Xiong$^{2}$\\
\name Xiaoze Liu$^{3}$\quad Ruowang Zhang$^{3}$\quad Keren Zhou$^{4}$\quad Kai Guo$^{1\dagger}$\\
\addr $^1$Michigan State University\quad $^2$Georgia Institute of Technology\quad $^3$Purdue University\quad
$^4$George Mason University\\
\addr $^*$Equal contribution\quad $^\dagger$Corresponding author}

\begin{document}
\raggedbottom

\maketitle

\begin{abstract}
Relational databases hold most enterprise data, and predicting on them has long required hand-written feature joins and
a model trained for each task. Relational foundation models (RFMs) remove this per-task engineering by predicting on an unseen database from a few
labelled examples in one forward pass. However, the strongest RFMs are closed, and open RFMs pre-trained on synthetic
data lag far behind. Prior work attributes
this gap to the lack of real pre-training data. We find that much of this gap comes from an incomplete synthetic prior and from a sparse
context. Existing synthetic priors lack history-dependent labels, so models pre-trained on them do not learn the
temporal and relational dynamics behind real tasks. A bounded
context also often misses the evidence that a task needs. We present \openrfm{}, an open recipe for building an RFM
from synthetic databases alone, whose data generator, model and training code are all open. \openrfm{} has three
components. (1)~A relational adapter, pre-trained on synthetic event streams with history-dependent labels,
encodes the sampled context of each entity into a feature row and performs in-context learning along the relation
axis, over the history of the entity and its related entities. (2)~A frozen tabular foundation model predicts from
labelled support rows and performs in-context learning along the batch axis. (3)~An optional harness, which an LLM
agent builds for each task, rewrites the schema so that the evidence a task needs enters the bounded context. On 32
benchmark tasks, \openrfm{} achieves the best mean rank, 3.00 against 3.63 for RelAgent and 5.34 for KumoRFM-2.
The code is available at \url{https://github.com/Quant-Spirit/openrfm}.
\end{abstract}

\input{sections/introduction}
\input{sections/preliminaries}

\input{sections/method}
\input{sections/experiments}

\input{sections/related}
\input{sections/conclusion}

\bibliography{references}
\bibliographystyle{tmlr}

\appendix
\section{Model and pre-training}
\label{app:impl}

\begin{algorithm}[!htb]
\caption{Prediction and pre-training with \openrfm{}}
\label{alg:openrfm}
\small
\begin{algorithmic}[1]
\Require database $\mathcal{D}$, labelled support rows $\mathcal{S}$, query rows $\mathcal{Q}$, adapter $\phi_\theta$, memory module $r_\psi$, frozen tabular model $g$
\Function{Predict}{$\mathcal{D},\mathcal{S},\mathcal{Q}$}
  \For{each task row $(e,t)\in\mathcal{S}\cup\mathcal{Q}$}
    \State $C_{e,t}\gets\sigma(\mathcal{D}_{<t},e,t)$ \Comment{at most $L$ cells, history labels only}
    \State $x_{e,t}\gets\big[\phi_\theta(C_{e,t}),\,r_\psi(C_{e,t},H_{e,t})\big]$ \Comment{relation axis, adapter row and memory columns}
  \EndFor
  \State \Return $g\big(\{x_q\}_{q\in\mathcal{Q}}\mid\{(x_k,y_k)\}_{k\in\mathcal{S}}\big)$ \Comment{batch axis, combined over $E$ support draws}
\EndFunction
\Repeat \Comment{pre-training of the adapter}
  \State draw a world $W\sim\Pi$, a target, a length $L$ and a support size $K$
  \State draw $\mathcal{S}$ and $\mathcal{Q}$ from the task rows of $W$, and mask random cells of each context
  \State $\mathcal{L}_{\mathrm{cell}}\gets$ masked-cell loss, $\;\mathbf{u}\gets\nabla_\theta\mathcal{L}_{\mathrm{cell}}/\lVert\nabla_\theta\mathcal{L}_{\mathrm{cell}}\rVert$
  \If{the update is coupled (every fourth)} \Comment{through the frozen $g$}
    \State $\mathcal{L}_{\mathrm{ICL}}\gets$ loss of $g$ on the adapter rows $\phi_\theta(C)$ of $\mathcal{S}$ and $\mathcal{Q}$, $\;\mathbf{u}\gets\mathbf{u}+{w_t}\,\nabla_\theta\mathcal{L}_{\mathrm{ICL}}/\lVert\nabla_\theta\mathcal{L}_{\mathrm{ICL}}\rVert$ \Comment{$w_t$ ramps up from zero}
  \EndIf
  \State update $\theta$ with $\mathbf{u}$
\Until{the step budget ends}
\end{algorithmic}
\end{algorithm}

\subsection{Model}
\label{app:model}

\textbf{Sampler.} The default sampler $\sigma$, shown in \hyperref[fig:pipeline]{Figure~\ref*{fig:pipeline}b}, builds the context of a task row $(e,t)$ from
$\mathcal{D}_{<t}$ by a breadth-first walk from the task row. It follows the bounded-width breadth-first search of
RT~\citep{ranjan2025relational}, which RT-J keeps for local expansion around rows that it ranks by random walks~\citep{ranjan2026rtj}.
As in RT, each visited row brings all its parent rows through its foreign keys. RT keeps up to a fixed number of child rows
drawn uniformly at random, while our sampler weighs them by time. The child rows of a visited row $r$ before the cutoff, $N_{<t}(r)=\{r'\in\mathcal{D}_{<t}: r\in\mathrm{Out}(r')\}$, are kept
whole when there are at most $A$ of them, and otherwise reduced to
\begin{equation}
\label{eq:sampler}
\sigma_A(r) = \mathrm{newest}_{\lfloor A/2\rfloor}\big(N_{<t}(r)\big)\ \cup\
\mathrm{strat}_{\lceil A/2\rceil}\big(N_{<t}(r)\setminus\mathrm{newest}_{\lfloor A/2\rfloor}(N_{<t}(r))\big),
\end{equation}
where $\mathrm{newest}_k$ keeps the $k$ most recent rows and $\mathrm{strat}_k$ draws $k$ of the remaining rows stratified on
the logarithm of their age. Half of each budget thus goes to the recent events that most tasks weigh, and the other half
covers the older history at a resolution that falls with age. The next hop expands only the retained rows, so the cost does
not grow with the degree of a hub. The walk stops when the context holds $L$ cells. Earlier task rows of $e$ enter like any
child row, with a label only when its window ends before $t$. The sampler runs on the GPU for a whole batch of task rows at
once (Appendix~\ref{sec:infra}). No graph sampler is best for every task. RT-J tunes its context hyperparameters per task, and
the selected values differ across tasks~\citep{ranjan2026rtj}. The sampler is therefore part of the harness and can be tuned for
each task. Our harness mainly rewrites the schema, so that one default sampler reaches the evidence that each task
needs.

\textbf{Adapter.} The adapter has $B=6$ blocks with hidden size $d=128$ and $4$ heads per layer, about 3.5 million
parameters. Of the 4 heads of the entity layer, two are mean heads and two are sum heads. The readout takes
the target cell after blocks 3 and 6, so with the memory module the learned part of the row has $2\times128+2Q=384$ columns.
Besides these learned columns, the row passes the raw numeric cells $v_c$ of the task row and of the entity row, except the hidden target cell, to the TFM
backbone, each next to its signed logarithm $\operatorname{sign}(v_c)\log(1+|v_c|)$. The TFM backbone accepts up to 500 input
columns. A harness configuration can also add a low-dimensional projection of each column state.

\textbf{Memory module.} The memory module has $Q=64$ queries with one head each of width 32, about 275{,}000 parameters,
and writes $2Q=128$ columns. Its context tokens are the cell states at all seven depths, from the block input to the
last block. Its history holds the newest $N=256$ rows within five foreign-key hops of the entity. History rows are
encoded in fixed blocks of 256, so the token of a row does not depend on which other rows are read.

\subsection{Training}
\label{app:training}

\textbf{Adapter pre-training.} The adapter uses Muon at learning rate $5\times10^{-4}$ for matrix parameters and AdamW at
$3\times10^{-4}$ for the rest, with 300 warm-up steps, a constant rate up to step 30{,}000 and cosine decay to 10\% of the
peak at step 50{,}000. Each update accumulates 262{,}144 masked cells, and each episode draws $L\in\{256,1024,4096\}$. The
in-context loss enters on every fourth update, and its weight ramps up over the first 300. We use the checkpoint at update
40{,}000. 

\textbf{Memory-module training.} Stage~2 trains $r_\psi$ alone for 2{,}500 updates with AdamW at learning rate
$3\times10^{-3}$, weight decay 0.01 and gradient clipping at norm 1, on one B200 GPU. Each episode holds contexts of
1{,}024 cells, 128 queries and a support set of 128, 256, 512 or 1{,}024 rows, all earlier than the queries. The frozen
TFM backbone is Kumo-Tabular small, and the memory loss $\mathcal{L}_{\mathrm{mem}}$ is its only loss. The number of
updates is set once for all tasks. Appendix~\ref{app:memworlds} describes its training worlds.

\subsection{Infrastructure}
\label{sec:infra}

Inference samples one context for every support and query row, so its speed depends on the sampler and the attention
layers as much as on the adapter. This section describes both components and the inference path built on them.

\textbf{GPU context sampler.} We implement a GPU-based context sampler in CUDA.
\begin{itemize}
\item Tables stay on the GPU as typed columns, and foreign keys as adjacency lists sorted by time, so a binary search
finds the rows before a cutoff without copying the database.
\item One kernel walks the graph for all task rows. Each row keeps a fixed budget of its newest children plus
log-time stratified samples, so cost does not grow with the degree of a hub.
\item Gather kernels write the adapter input tensors directly with no host round trip. Seeded Philox draws make a
context reproducible, and CUDA graphs remove the launch overhead of the many small kernels.
\end{itemize}

\textbf{Sparse attention.} The dominant attention pattern of the relational layers is sparse. Each cell attends only
within a small group, such as one column, one row with its parent rows, or the child rows of a row, so a dense masked
attention wastes most of its work. We implement a Triton kernel for this pattern, written and optimised with Claude
Opus 5.5. Its two main design choices follow. One stable sort per attention type turns every group into a contiguous
run, so each tile of 16 query cells attends to a few contiguous key ranges and no dense mask is built. The plan uses only
fixed-shape device operations and reads nothing back to the host, and the backward pass needs no atomic operations and is deterministic. The kernel matches dense attention up to floating-point rounding.
On driver-dnf it lowers the attention device time from 2.48~s to 1.65~s. Its plan costs about 3.5~ms per forward pass,
so it pays off once each forward pass holds many rows.

\textbf{Inference path.} Sampling and data loading take most of the inference time. The row graph stays resident across
tasks and splits, the tables are memory-mapped, and the adapter blocks are compiled and run in bfloat16. Host
prefetching and a second CUDA stream overlap data loading with the forward pass. We further improve the inference
efficiency with CUDA graphs.

\textbf{Comparison with RT-J.} Table~\ref{tab:latency} reports the wall-clock time of all predictions of one task on one
B200 GPU for \textsc{rel-f1} driver-position, driver-dnf and driver-top3, as the median of three CUDA-synchronised runs
after one warm-up pass. Averaged over the three tasks, the default version of \openrfm{} takes 25.4~s against 16.9~s for
RT-J, or 1.5 times as long. With the help of CUDA graphs, it takes 11.7~s and is 1.4 times faster than RT-J, and faster
on each of the three tasks. This comparison does not aim to show that \openrfm{} is more efficient than RT. \openrfm{}
is more complex than RT. Its architecture is GPU friendly, however, so even a complex structure can be tuned further
with dedicated kernels, as the sparse kernel and the CUDA graph version of this section show, and a coding agent can
carry out much of this tuning automatically. As long as an architecture is expressive and runs well on the GPU, its
inference cost keeps falling with such tuning. A model should therefore possibly not give up expressive power for
inference speed.

\input{tables/latency}
\FloatBarrier

\section{Synthetic prior}
\label{app:prior}

All pre-training data of \openrfm{} is synthetic and drawn from the prior described in this appendix. One world is a
relational database that consists of entity tables, hub and ancestor tables, event tables, a label history and a task
table. The generative process creates all event streams first and draws the task afterwards. Because labels depend on
the past activity of an entity and its peers under shared hubs, pre-training forces the model to track both temporal
and relational dynamics.

\label{app:plurel}Table~\ref{tab:prior_schema} compares the schema, event
dynamics and targets of real RelBench databases, PluRel and our prior. PluRel draws random DAG schemas with light-tailed fan-out,
evenly spaced events, rare cold-start entities and static attribute targets. Our prior follows the real databases more
closely, with shared hubs and ancestors, heavy-tailed fan-out, bursty and seasonal event times, cold-start entities and
future-window targets next to autocomplete.

\begin{table}[!htb]
\centering
\caption{Properties of real RelBench databases, PluRel and our prior.}
\label{tab:prior_schema}
\small
\begin{tabular}{@{}llll@{}}
\toprule
Property & Real RelBench & PluRel & Ours \\
\midrule
Schema & stars and snowflakes with hubs & random DAGs & shared hubs and ancestors \\
Fan-out & heavy tail & light tail & heavy tail \\
Event times & bursty & evenly spaced & bursty and seasonal \\
Cold-start entities & common & rare & present \\
Targets & mostly future windows & static attribute functions & future windows, autocomplete \\
\bottomrule
\end{tabular}
\end{table}

\subsection{Generating a world}
\label{app:worldgen}

Each world contains 4 to 10 tables. Every entity participates in one to three event
relations, with an expected event count per entity drawn log-uniformly between 2 and 256.

Algorithm~\ref{alg:worldgen} formalises the generative process of the future-activity and autocomplete worlds. The event
streams are fixed before the task is drawn, so one world can carry different tasks without changing any event.
$\mathcal{U}$ denotes a uniform draw and $\log\mathcal{U}$ denotes a draw that is uniform on the log scale. The numeric
mark of an event is
\begin{equation}
\label{eq:mark}
m_j = w_0x_e+w_1x_{p(e)}+w_2x_{a(p(e))}+w_3\bar x_{\mathrm{peer}(e)} + w_4\big(u_e+(d_e+\vartheta\,d_{p(e)})(t_j/T-1/2)\big)+\epsilon_j ,
\end{equation}
where $\bar x_{\mathrm{peer}(e)}$ is the mean attribute of the peers, $w$ is a sparse signed weight vector drawn once per
world and $\epsilon_j$ is Gaussian noise with an entity-specific scale. The aggregates $\mathcal{F}$ of a future-window
task are existence, count above a threshold, count, sum, mean, maximum, the category of the next event and the time to
the next event. Each entity table has 1 to 32 attribute columns, and a text column is added with probability 0.2.
The activity $a_e(t)$ is 1 in an active cell and $\delta$ in a dormant one, and each of the $G$ cells of the time grid has width $\Delta t=T/G$. Table~\ref{tab:worldgen} lists the constants of \eqref{eq:intensity} and the ranges from which each world draws its
parameters.

\begin{algorithm}[!tb]
\caption{Generating one world and its task}
\label{alg:worldgen}
\small
\algnewcommand{\StepHead}[1]{\Statex \textcolor{black!60}{\textsc{#1}}}
\begin{algorithmic}[1]
\Require world index $k$, constants and ranges of Table~\ref{tab:worldgen}
\Ensure relational database $\mathcal{D}$ with event tables, label history and task table
\StepHead{Step 1. Structure and latent states}
\State $n\gets192$, \quad $n_h\sim\mathcal{U}\{4,\dots,24\}$, \quad $n_a\gets\max(2,\lfloor n_h/3\rfloor)$
\State $p(e)\sim\mathcal{U}\{1,\dots,n_h\}$, \quad $a(h)\sim\mathcal{U}\{1,\dots,n_a\}$ \Comment{entities with one hub are peers}
\State draw $T,\bar n,\vartheta,\chi,b,\omega,\beta,P$$,H$ from Table~\ref{tab:worldgen}, \quad $R\sim\mathcal{U}\{1,2,3\}$
\State $x_e,\,x_h,\,x_a,\,\eta_h,\,\epsilon_e,\,\epsilon'_e\sim\mathcal{N}(0,1)$
\State $u_e\gets\sqrt{\vartheta}\,\eta_{p(e)}+\sqrt{1-\vartheta}\,\epsilon_e$, \quad $\lambda_e\gets\chi\,u_e+\sqrt{1-\chi^2}\,\epsilon'_e$
\State $d_e\sim\mathcal{N}(0,0.8^2)$, \quad $d_h\sim\mathcal{N}(0,0.7^2)$
\StepHead{Step 2. Event streams}
\For{$r=1,\dots,R$}
  \State draw $c_r$ from Table~\ref{tab:worldgen}
  \For{$e=1,\dots,n$}
    \State $a_e\gets$ two-state chain on $G$ cells with switch probabilities $q_{\mathrm{off}},q_{\mathrm{on}}$, phase $\varphi_e\sim\mathcal{U}[0,2\pi)$
    \State $\mathcal{E}\gets$ immigrants, $\mathrm{Poisson}\big((1-b)\,\mu_{e,r}\,\Delta t\big)$ per cell \Comment{\eqref{eq:intensity}}
    \Repeat
      \State each new event in $\mathcal{E}$ adds $\mathrm{Poisson}(b)$ children at delays $\sim\mathrm{Exp}(\omega)$
    \Until{no new event falls inside $[0,T]$}
    \State for $j\in\mathcal{E}$, set $m_j$ by \eqref{eq:mark} and set ${c_j=}\mathbf{1}\{m_j+\epsilon''_j>0\}$
  \EndFor
\EndFor
\StepHead{Step 3. Task, drawn after all streams from an independent seed}
\If{$k$ is even}
  \State $\tau\gets$ hide a numeric or categorical mark, or an entity attribute
\Else
  \State $r^\star\sim\mathcal{U}\{1,\dots,R\}$, \quad $f\sim\mathcal{F}$
  \State $\tau(e,t)\gets f\big(\{(t_j,m_j{,c_j})\in r^\star_e : t<t_j\le t+H\}\big)$ \Comment{\eqref{eq:future_target}}
\EndIf
\StepHead{Step 4. Label history}
\If{the task is a future window}
  \For{origins $t_0$ spaced by $\max(H,0.025T)$ over $[0.08T,\,0.96T-H]$}
    \State store $\tau(e,t_0)$ for every $e$ with timestamp $t_0+H$ \Comment{visible once known}
  \EndFor
\Else
  \State store the mark or category of every earlier event of $e$ at its own event time
\EndIf
\StepHead{Step 5. Cases}
\State $t^{(1)},\dots,t^{(4)}\gets$ evenly spaced from $\mathcal{U}[0.3T,0.4T]$ to $\mathcal{U}[0.7T,T-H]$
\State cases $\big(e,\,t^{(i)},\,\tau(e,t^{(i)})\big)$ for every $e$ and $i$, with supports from $t^{(1)}$ and queries from a later $t^{(i)}$ on disjoint $e$
\StepHead{Step 6. Export}
\State $\mathcal{D}\gets$ ancestor, hub and entity tables, $R$ event tables, label history and task table
\State draw a table count in $[R+3,10]$, merge the hierarchy into one table when it is below $R+5$, and fill the rest with unrelated tables
\State add a text column with probability 0.2
\State \Return $\mathcal{D}$
\end{algorithmic}
\end{algorithm}

\begin{table}[tb]
\centering
\caption{Constants and per-world ranges of the generative process. Each world draws one value from every range.}
\label{tab:worldgen}
\footnotesize
\setlength{\tabcolsep}{4pt}
\begin{tabularx}{\linewidth}{@{}l>{\raggedright\arraybackslash}Xll@{}}
\toprule
Symbol & Meaning & Value or range & Draw \\
\midrule
$\gamma$ & sensitivity of the rate to $s_e(t)$, exponent clipped to $[-4,4]$ & 1.3 & fixed \\
$\alpha,\ \alpha'$ & weights of entity and hub drift in the rate & 0.6, 0.4 & fixed \\
$G$ & cells of the time grid & 64 & fixed \\
$q_{\mathrm{off}},\ q_{\mathrm{on}}$ & activity switch probabilities per cell & 0.055, 0.13 & fixed \\
$\delta$ & rate factor of a dormant cell & 0.015 & fixed \\
\midrule
$T$ & time span & $[32,1024]$ & log-uniform \\
$\bar n$ & mean event count per entity & $[2,256]$ & log-uniform \\
$c_r$ & scale per relation & $[0.35,2.5]$ & log-uniform \\
$\vartheta,\ \chi$ & hub share of the level, coupling of rate and marks & $[0,1]$ & uniform \\
$b$ & branching ratio & $[0,0.65]$ & uniform \\
$\omega$ & mean offspring delay & $[0.002T,0.08T]$ & log-uniform \\
$\beta,\ P$ & season amplitude and period & $[0,0.95]$, $[0.04T,0.8T]$ & uniform, log-uniform \\
$H$ & horizon & $[0.025T,0.25T]$ & log-uniform \\
\bottomrule
\end{tabularx}
\end{table}

\subsection{Labels and data filtering}
\label{app:corpus}

\textbf{Label families.} The pre-training corpus draws its tasks from the ten label families of
Table~\ref{tab:targets_full}, a fine-grained split of the five label groups in the main text. \begin{enumerate}
\item \textit{Future window.} The label applies one aggregate from \{existence, count above a threshold, count, sum,
mean, next category, time to next event, maximum\} to the events of an entity in $(t,t+H]$.
\item \textit{Window activity.} The label is one of \{exact count, sum, rate change, last time, burst\} in a window of
raw dated rows.
\item \textit{Derived column.} The label is one of \{count, sum, mean, latest value, age\} over the past events of an
entity and its peers.
\item \textit{Event program.} The label is an SQL query over the same events.
\item \textit{Set or lookup.} The label is one of \{quantile, maximum, threshold\} over a set of rows, or a value
looked up in a parent.
\item \textit{Persistent history.} The label \{drifts, reverts, switches, repeats\} over time.
\item \textit{Autocomplete.} The label is a hidden value of an existing row.
\item \textit{Fact-row autocomplete.} The label is the outcome of a row with one or two changing parent rows.
\item \textit{Row or support function.} The label is a nonlinear function of row values, or a mapping given only by the
labelled examples.
\item \textit{Group choice.} The label marks which item of a set offered together is chosen.
\end{enumerate}

\input{tables/label_families}

\textbf{Data filtering.} We filter candidate worlds before adding them to the corpus. A candidate world is kept only if
a gradient-boosted model with 64 support rows reaches AUROC between 0.60 and 0.98, or $R^2$ between 0.05 and 0.95, on
features of its observable history. We group these features into six kinds of evidence, namely rates and means, counts,
recency, peer statistics, attributes and extremes. A kind is required if removing it lowers the score by at least 0.02
AUROC or 0.05 $R^2$. We steer sampling so that each kind is required by at least 5\% of kept worlds.

\subsection{Worlds for memory-module training}
\label{app:memworlds}
Stage~2 trains the memory module on synthetic worlds of its own, separate from the pre-training corpus. The two
modules learn different things. The adapter learns to read the context, which keeps the recent rows of an entity at full
resolution. The memory module only adds value where the evidence lies beyond that context, in long and irregular
histories. In worlds whose evidence fits in the context, the frozen adapter row already carries the signal, and the
memory module has nothing to learn. Its worlds therefore stretch histories over many rows, place the evidence far
from the cutoff or among distractor columns, and change the dynamics over time. Many worlds need no history at all, and the family sampler below draws them often, so that the memory module learns to stay silent when history does not help. Label programs keep its columns
compatible with the generated columns of a schema rewrite.

\textbf{Worlds.} These worlds come from a separate generator, simpler than Algorithm~\ref{alg:worldgen}. A world holds
one entity table and one event stream, and each event links to a hub drawn by Zipf popularity. History lengths are
log-normal across entities, each entity is active over its own time span and may stop early, and the marks carry state
through an entity effect, a slow drift and AR(1) noise. Each world family then defines its label as a signal over this
stream. Each world contains 256 entities. In a quarter of
the worlds the entities have long histories, with a median of 60 to 250 rows per entity and at most 1{,}500. A label is a
standardised score plus Gaussian noise whose standard deviation is drawn from $[0.1,0.7]$. For classification we threshold the score
at a quantile drawn from $[0.2,0.8]$. We use each world for about 40 consecutive updates.

\textbf{World families and labels.} Table~\ref{tab:targets_full} lists the eleven world families next to the label
families of pre-training. Each world carries one task from its family.
\begin{enumerate}
\item \textit{Memory signal.} The label is one function of the earlier rows of an entity, such as a count, an old mark,
recency or hub peers.
\item \textit{Exact read.} The label is the $k$-th newest value, an extreme or the days since the last event of one
column among 3 to 12 distractor columns.
\item \textit{Peer excitation.} Entities under one hub excite each other in a branching process, and the label is the
count over the next horizon.
\item \textit{Regime change.} Marks and rates change at a hidden time.
\item \textit{Censoring.} The label is the time to the next event of entities that leave at a hidden time, censored at
the horizon.
\item \textit{Heavy tail.} The label is $\exp\!\big(s(a_t+b_t r)\big)$ for one history read $r$, with level and slope
drifting over cutoffs.
\item \textit{Near cells.} The label needs only the task row, the entity row and its hub row.
\item \textit{Joint signal}, two families. The label combines history with these cells additively, by interaction or as
a deviation of the recent rate.
\item \textit{Replay.} The world reuses a stored task of a pre-training world.
\item \textit{Silence.} The world is pure noise, so any memory column that fits it costs query loss.
\end{enumerate}
Besides its own task, every world outside the near-cell and silence families adds up to three tasks per update defined
by label programs. They keep the memory module compatible with the generated columns of a schema rewrite, each of which
is an SQL program over the rows of an entity. A label program is a random query over the rows that an entity reaches
through foreign keys. It fixes a column, a predicate with a quantile threshold, a time window, a last-$k$ limit, a
foreign-key slot and a reducer from \{count, sum, mean, variance, minimum, maximum\}. Half of all programs count rows,
and the programs fall into 13 groups that range from plain aggregates to recency, trends and distinct linked keys.
The three tasks are one program over the past rows, one over the future window $(t,t+H]$ with $H$ log-uniform up to
365 days, and one column of 256 SQL queries that an LLM writes for the world. Programs only define labels and never
enter the input. A label passes through a signed logarithm and becomes a regression target or a binary label split at
the support median with equal probability.

\textbf{Family sampler.} We draw the family of the next world in proportion to the loss that the memory module can still
remove on it. After the updates of a world of family $f$, we fit ridge models on the rows of its task and compute
\begin{equation}
\label{eq:reducible}
v = \max\Big(0,\ \mathcal{L}(\phi_\theta+r_\psi)-\min\big\{\mathcal{L}(\phi_\theta+\text{oracle}),\ \mathcal{L}(\phi_\theta)\big\}\Big),
\qquad \rho_f\gets\rho_f+0.1\,(v-\rho_f),
\end{equation}
where $\mathcal{L}$ is the query loss on a unit-variance label, $\phi_\theta+r_\psi$ is the adapter row with the memory
columns, and the oracle appends 128 label programs evaluated exactly on the history of the entity. We draw the next family
with probability $\rho_f/\sum_{f'}\rho_{f'}$, and a new family starts at the mean of the others. Families on which the memory row still trails the oracle or the adapter row alone are drawn more often. They include worlds where history does not help, on which the memory module learns to add nothing. The sampler draws 64 worlds, and
near-cell and censoring worlds take 84\% of them, as Table~\ref{tab:targets_full} shows.

{\let\FloatBarrier\relax\section{Harness and evaluation protocol}\label{app:protocol}}

This appendix specifies the default configuration and the evaluation protocol used across all 32 benchmark tasks. It also details the LLM harness search.

\subsection{Default configuration}
\label{app:default}

Without a task-specific search, \openrfm{} runs one default configuration, which the settings of Appendix~\ref{app:eval} start from. The database keeps its original schema with
no rewrite program and no generated columns. A breadth-first sampler walks from the task row over rows before the
cutoff, following references to parent rows and child rows, until the context reaches its cell budget (default $L=1{,}024$). The support set
draws rows uniformly at random from rows whose labels are already known at the cutoff, or keeps all such
rows when fewer exist (default $K=16{,}384$). Each prediction combines several draws of the frozen Kumo-Tabular small TFM backbone (default $E=4$).
Classification takes the mean across draws, and regression takes the median with untransformed targets.

\subsection{Evaluation protocol}
\label{app:eval}

All settings share one set of weights across all 32 tasks, consisting of the adapter at update 40{,}000, the memory module at update 2{,}500 and the frozen Kumo-Tabular small TFM backbone. Per task, the two harness settings tune the following.
\begin{itemize}
\item \textbf{Original schema}, the R1 setting, keeps the schema and the breadth-first sampler of Appendix~\ref{app:default} and chooses $L\in\{1{,}024,2{,}048,4{,}096\}$, $K$ from 512 to 65{,}536, $E$ and the support pool, leaving regression targets untransformed.
\item \textbf{LLM-searched harness} chooses the schema program or the original schema, $K$, $E$, the support order, the target transform and $L$.
\end{itemize}

\textbf{Ablation subset.} Ablations use 9 tasks, 6 for classification and 3 for regression. They are \textsc{rel-avito} user-clicks, user-visits and ad-ctr, \textsc{rel-f1} driver-dnf, driver-top3 and driver-position, \textsc{rel-trial} study-outcome, \textsc{rel-event} user-ignore and \textsc{rel-amazon} item-ltv.

\subsection{LLM harness search}

\textbf{Schema programs.} A schema program is a JSON document that keeps the original tables, leaves the task table unchanged and adds generated tables reached by at least one foreign-key hop. These generated tables include a snapshot table at entity $\times$ as-of grain built from SQL over rows before the cutoff and history labels, an evidence table of the last $k$ original rows before each cutoff and derived event or pair tables capped at 64 rows per case. The SQL sees only original tables, unlabelled cases and history labels.

\textbf{Proposal loop.} The agent receives the schema and the task, proposes a harness configuration and receives as reward a combination of the score of the frozen \openrfm{} model with that harness and the share of labels per context. Each reply is validated and dry-run, returning errors to the agent for at most three repairs. Test labels never reach the agent. A rewrite is chosen on 23 tasks and the original schema on 9.

\end{document}

%% file: sections/introduction.tex
\section{Introduction}
\label{sec:intro}

Relational databases~\citep{codd1970relational} remain the dominant way to store enterprise data, and they underpin
applications ranging from e-commerce~\citep{ni2019amazonreviews} and finance~\citep{motl2015ctu} to
healthcare~\citep{johnson2016mimic}. Such data carries two levels of structure, the schema within each table and
the primary key and foreign key references across tables. Predicting on it has traditionally required task-specific models such as XGBoost~\citep{chen2016xgboost},
fed by hand-written feature joins and tuned for each task, a pipeline that is labour-intensive and hard to scale.
Relational deep learning~\citep{fey2024rdl} learns the aggregation from data instead, but it still trains one model per
task.

Relational foundation models (RFMs) eliminate this per-task training with one pre-trained model. The model
receives an unseen database, a predictive task defined on it, such as whether a user will churn in the next month,
and labelled examples of that task, and it returns predictions in one forward pass through
relational in-context learning~\citep{fey2025kumorfm,hudovernik2026kumorfm}. Closed commercial RFMs
like KumoRFM-2~\citep{hudovernik2026kumorfm} perform well, but their weights and data are closed.
The strongest open alternatives flatten the database for a tabular predictor. RDBLearn~\citep{zhang2026rdblearn}
combines deep feature synthesis (DFS) aggregations~\citep{kanter2015dfs} with a tabular foundation
model, while RelAgent~\citep{huang2026relagent} uses an LLM agent to write SQL features and trains a
gradient-boosted model per task. The former requires DFS to generate a heavy set of features for
every task, which makes it CPU- and I/O-heavy and task-agnostic, while the latter requires
long-horizon agent search and per-task training. Neither has learnable relational parameters, which
makes further pre-training or fine-tuning difficult. Open parametric RFMs such as RT~\citep{ranjan2025relational}
and PluRel~\citep{kothapalli2026plurel} remain far behind in performance, and there is no public recipe
for building a native relational foundation model with frontier performance.

\input{figures/pipeline}

We introduce \openrfm{}, an RFM that reaches frontier predictive performance while being open and pre-trained purely on
synthetic data. Tabular foundation models (TFMs) advance rapidly~\citep{jager2026tabpfn35,qu2026tabiclv2,nvidia2026kumotabular}
and already solve many relational tasks on aggregated features~\citep{zhang2026rdblearn}. What these features lack is a more
expressive relational signal, so \openrfm{} attaches a learned relational adapter to a frozen TFM (Figure~\ref{fig:pipeline}). The architecture has three components. (1)~A \emph{relational adapter} encodes the sampled context of each entity into
a feature row, and a memory module adds columns from the history that the context cannot hold. Both are pre-trained on
synthetic databases alone. (2)~A frozen \emph{tabular foundation model} serves as the predictive backbone. (3)~An
optional \emph{LLM agent} improves the harness of the model by deciding which data enter its context. This architecture can directly perform
relational in-context learning or be further fine-tuned to better adapt to downstream tasks. Three key design choices
make this architecture work well.

\noindent\textbf{(1) A prior with temporal and relational dynamics.} A prior-fitted model can only use what its synthetic
prior makes informative~\citep{muller2022pfn,xie2022icl}. In previous synthetic priors~\citep{kothapalli2026plurel},
labels are static functions of nearby attributes. Without history-dependent labels, the events of an entity or of
related entities carry little information about the label and the learned predictor ignores them, which is why RT still needs real
data~\citep{ranjan2025relational}. Our prior generates databases as event streams in which entities share latent states
through hubs, and defines labels from the events of an entity and its peers, many of them as aggregations of future events (Section~\ref{sec:data}).

\noindent\textbf{(2) Two-axis in-context learning.} The relational adapter performs in-context learning along the relation axis, inferring each entity's latent state from its relational context, including its history labels and those of related entities. The tabular foundation model performs in-context learning along the batch axis, inferring the task's mapping from labelled support rows. The adapter is trained through the frozen backbone's in-context loss (Section~\ref{sec:data}).

\noindent\textbf{(3) Context-shaping harness.} An RFM attends over a bounded context of at
most $L$ cells, and evidence that the sampler leaves out never reaches it. When that evidence lies many hops away, the
budget runs out on intermediate rows, so what enters the context must be decided before the model runs. We find this a
major bottleneck. An LLM agent builds this harness by rewriting the schema. On \textsc{rel-trial} study-outcome, a rewrite that links each
study to earlier studies sharing its sponsor, condition or intervention cuts the contexts without a history label from
74.3\% to 1.7\% (Figure~\ref{fig:evidence:context}).

On 32 classification and regression tasks from RelBench-V1, RelBench-V2 and 4DBInfer, as adopted by
\citet{huang2026relagent}, \openrfm{} achieves the best mean rank, 3.00 against 3.63 for RelAgent and 5.34 for
KumoRFM-2. We release the data generator, the model and the training code as a public recipe for future research on
relational foundation models.

%% file: figures/pipeline.tex
\begin{figure}[t]
\centering
\includegraphics[width=\linewidth]{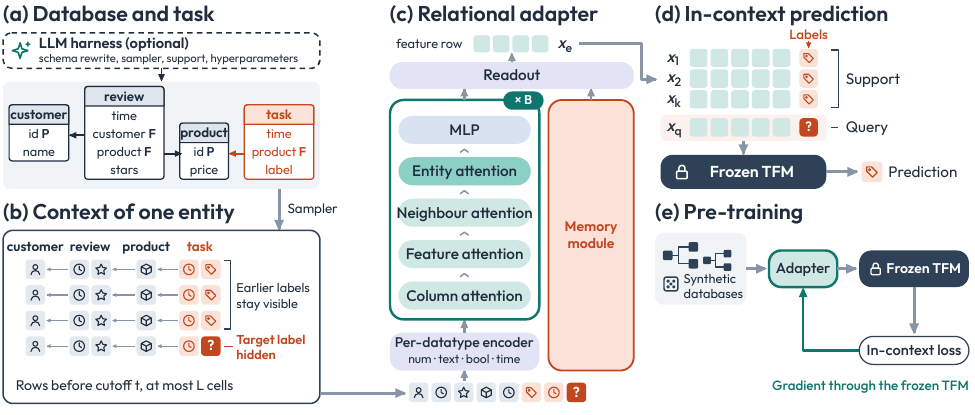}
\caption{\textbf{\openrfm{}.} (a) An optional LLM harness sets the schema rewrite, sampler, support rule and
hyperparameters per task (Section~\ref{sec:harness}). (b) The sampler builds a bounded context from rows before the cutoff,
with earlier history labels visible and the target hidden. (c) The adapter and a memory module over the context and the history
write one feature row per entity (Section~\ref{sec:model}). (d) A frozen TFM predicts from these rows and the support
labels. (e) The adapter is pre-trained on synthetic databases through the frozen TFM (Section~\ref{sec:data}).}
\label{fig:pipeline}
\end{figure}

%% file: sections/preliminaries.tex
\section{Preliminaries}
\label{sec:prelim}

\subsection{Relational databases and predictive tasks}
\label{sec:prelim:task}

\textbf{Relational database.}
Beyond a storage format, a relational database is a universal way to model structured data about the
world~\citep{zhang2026worldmodeling}. A relational database (RDB) $\mathcal{D}=\{T_1,\dots,T_m\}$ is a set of tables. A row $r\in T_j$ has a primary key
$\mathrm{pk}(r)$, zero or more foreign keys referencing primary keys in other tables, and a set
of attribute cells. Each cell $(r,c,v)$ holds the value $v$ of column $c$ in row $r$, with a semantic type such as
number, category, text, datetime, or boolean. A row may also carry a timestamp $t_r$ that records the time when
it becomes known. Tables belong to two categories. \emph{Fact tables} record timestamped events, such as the reviews in
\textsc{rel-amazon}, each
linking a customer to a product at a given time. \emph{Dimension tables} describe entities, such as customers and
products, and usually carry no timestamp. We write $\mathcal{D}_{<t}$ for the
database as it exists before time $t$. It keeps from every table the rows $r$ with $t_r<t$ and all rows without a timestamp. PK--FK references form a
\emph{schema graph} over tables and a \emph{row graph} over individual rows.

\textbf{Predictive task.}
Following RelBench~\citep{robinson2024relbench}, a task is defined by an \emph{entity table}
$T_{\mathrm{ent}}\in\mathcal{D}$ and a \emph{task table}
$\mathcal{T}=\{(e_i,t_i,y_i)\}_{i}$. Each row asks for the label $y_i$ of entity $e_i\in T_{\mathrm{ent}}$ at cutoff
time $t_i$. The label $y_i$ is categorical for classification tasks and real-valued for regression tasks.
For a \emph{forecasting} task, $y_i$ is a function of the events involving $e_i$ in the future window
$(t_i,t_i+H]$ with horizon $H$. For example, \textsc{rel-amazon} user-churn asks whether a user writes no review in the
next three months. For an \emph{autocomplete} task, $y_i$ is a missing attribute of a row that already exists at
$t_i$. For example, the ctr task on \textsc{outbrain} in 4DBInfer~\citep{wang2024fourdbinfer} asks whether a displayed
recommendation is clicked, a value missing from that display row. The task table is split by time into
$\mathcal{T}_{\mathrm{train}},\mathcal{T}_{\mathrm{val}},\mathcal{T}_{\mathrm{test}}$, with all cutoffs of a later
split placed after those of an earlier one.

\begin{table}[t]
\centering
\caption{Relational prediction methods. \emph{Learned relational encoder} means that relational evidence is
summarised by trained parameters. \emph{TFM backbone} means that a pre-trained tabular foundation model makes the final prediction.
\emph{Training-free} means that no model is trained on the target task.}
\label{tab:rfm_compare}
\small
\setlength{\tabcolsep}{4pt}
\resizebox{\linewidth}{!}{%
\begin{tabular}{@{}llccccc@{}}
\toprule
Method & Relational evidence & Pre-training data & \makecell{Learned rel.\\encoder} & \makecell{TFM\\backbone} &
\makecell{Training-\\free} & Open \\
\midrule
RDBLearn~\citep{zhang2026rdblearn} & DFS features & none (TFM only) & $\times$ & \checkmark & \checkmark & \checkmark \\
RDB-PFN~\citep{wang2026relational} & DFS features & synthetic & $\times$ & $\times$ & \checkmark & \checkmark \\
RelAgent~\citep{huang2026relagent} & LLM-written SQL & none & $\times$ & $\times$ & $\times$ & \checkmark \\
RT~\citep{ranjan2025relational} & cell sequence & real & \checkmark & $\times$ & \checkmark & \checkmark \\
RT-J~\citep{ranjan2026rtj} & cell sequence & synthetic + real & \checkmark & $\times$ & \checkmark & \checkmark \\
PluRel~\citep{kothapalli2026plurel} & cell sequence & synthetic & \checkmark & $\times$ & \checkmark & \checkmark \\
KumoRFM-2~\citep{hudovernik2026kumorfm} & row subgraph & synthetic + real & \checkmark & $\times$ & \checkmark & $\times$ \\
\midrule
\openrfm{} (ours) & cell sequence (+ rewrite) & synthetic & \checkmark & \checkmark & \checkmark & \checkmark \\
\bottomrule
\end{tabular}}
\end{table}

\subsection{Relational foundation models}
\label{sec:prelim:rfm}

A relational foundation model (RFM) is pre-trained once and applied to new databases and tasks without training.
For a query $(e_\star,t_\star)$, it receives a \emph{support set} $\mathcal{S}=\{(e_k,t_k,y_k)\}_{k=1}^{K}$ of $K$ task
rows whose labels are known before $t_\star$, meaning their future windows end before the query cutoff, and returns
\begin{equation}
\label{eq:rfm}
\hat y_\star = f\big(e_\star,t_\star \,\big|\, \mathcal{S},\ \mathcal{D}_{<t_\star}\big)
\end{equation}
in one forward pass~\citep{fey2025kumorfm}. The labelled examples in $\mathcal{S}$ specify the task through
relational in-context learning, which requires the model to transfer across schemas, column meanings, and label spaces
absent from pre-training. An RFM does two things. It must summarise the relational evidence for each
entity, and it must infer the query label from the labelled examples. Existing RFMs follow two routes.
Table~\ref{tab:rfm_compare} compares them with our model, \openrfm{}.

\subsubsection{Flattening into a tabular foundation model}
\label{sec:prelim:flatten}

The first route keeps the relational structure fixed and reuses a model built for single tables.

\textbf{Tabular foundation models.} A TFM solves \eqref{eq:rfm} for a single table. Given support rows
$\mX_{\mathcal{S}}\in\R^{K\times d}$ with labels $\vy_{\mathcal{S}}$ and query rows $\mX_{\mathcal{Q}}$, it returns
$g(\mX_{\mathcal{Q}}\mid\mX_{\mathcal{S}},\vy_{\mathcal{S}})$ in one forward pass. TabPFN~\citep{hollmann2025tabpfn} and
TabICL~\citep{qu2025tabicl} are pre-trained on synthetic tables from a prior over data-generating processes, so the
forward pass approximates Bayesian prediction under that prior~\citep{muller2022pfn}. TabICL embeds each column, mixes
features within each row, and lets query rows attend to labelled support rows. Recent models scale to larger tables and
stronger priors~\citep{jager2026tabpfn35,qu2026tabiclv2,nvidia2026kumotabular}.

\textbf{Turning a database into a table.} A TFM needs one fixed-width row per entity. RDBLearn~\citep{zhang2026rdblearn}
computes deep feature synthesis~\citep{kanter2015dfs}, which recursively applies aggregates such as count, mean and
maximum along foreign-key paths over the rows before each cutoff, and feeds the resulting table to a TFM.
RelAgent~\citep{huang2026relagent} uses an LLM agent to explore the database and to write and validate SQL features,
and then trains a gradient-boosted model per task.

This route inherits the strength of tabular models, and its weakness is table construction. DFS misses joint conditions
on several columns of a related row~\citep{hudovernik2026kumorfm}, its feature count grows with the depth and fan-out of
the schema~\citep{huang2026relagent}, and no expansion depth reaches task labels over future windows, which are not in
the database. RelAgent avoids these limits with task-specific SQL, but it needs a long program search and a new model for
each task. In both pipelines the features come from a SQL engine on the CPU that no gradient reaches.

\subsubsection{Native relational models}
\label{sec:prelim:native}

The second route learns both components in one network that reads the relational structure directly.

\textbf{Relational Transformer and KumoRFM.} Both models read a sampled neighbourhood of the query entity.
RT~\citep{ranjan2025relational} treats every cell of this neighbourhood as a token. Support label cells sit in the same
sequence, so relational aggregation and in-context learning share the same layers. RT is pre-trained on real RelBench
databases, and RT-J~\citep{ranjan2026rtj} scales it with longer contexts, a label-aware sampler and pre-training on
PluRel and real web-crawled databases. KumoRFM~\citep{fey2025kumorfm} works on rows and lets the query attend to the
subgraphs of labelled support entities. KumoRFM-2~\citep{hudovernik2026kumorfm} injects support labels early and is
pre-trained on synthetic and real databases. It reaches frontier performance on RelBench, with few technical details
disclosed.

\textbf{Pre-training data.} Griffin~\citep{wang2025griffin} pre-trains first on single tables and then on RDBs, and
RDB-PFN~\citep{wang2026relational} pre-trains only on synthetic relational data. Neither approaches the frontier. RT
pre-trained only on PluRel~\citep{kothapalli2026plurel} databases falls behind RT pre-trained on real databases on most
tasks, and RT-J improves by adding real databases after PluRel. These works conclude that real data is necessary, but
they do not explain why synthetic data falls short. Section~\ref{sec:challenges} compares the two routes on real tasks
to locate where each of them fails.

%% file: sections/method.tex
\section{OpenRFM}
\label{sec:method}

This section presents the design of \openrfm{} (Figure~\ref{fig:pipeline}). We start with a comparison of existing
methods on real tasks, which shows where each of them fails and yields three requirements for a good relational
predictor (Section~\ref{sec:challenges}). We then describe the model (Section~\ref{sec:model}) and the synthetic prior
it is pre-trained on (Section~\ref{sec:data}), both designed to meet these requirements. Finally, we present the
harness, which brings into the context the evidence that the model cannot reach by itself (Section~\ref{sec:harness}).

\subsection{Adapter-based RFM with two-axis in-context learning}
\label{sec:adapter}

\subsubsection{What makes a good relational predictor}
\label{sec:challenges}

We begin by comparing four representative methods on all RelBench classification tasks, RDBLearn, RT, RT pre-trained on PluRel and
KumoRFM-2. RDBLearn and RT follow the two routes of Section~\ref{sec:prelim:rfm}, and RT and RT-PluRel share one model
pre-trained on real and on synthetic databases. We describe each task by two model-agnostic metrics. The \emph{label persistence} $\lambda=I(Y;S\mid C)/H(Y)$ measures how much the last label
$S$ of an entity whose window ends before the cutoff tells about its next label $Y$ beyond its activity state $C$, given
by quartiles of recent event count, total event count and recency. The \emph{own-label density} $\rho_{\text{own}}$ is
the fraction of queries whose entity has at least one earlier label known before the cutoff.

Table~\ref{tab:prelim_study} shows the four tasks where the methods differ most, and the two metrics explain most of the
differences. Where earlier labels are dense and persistent, as on the \textsc{rel-f1} tasks, RT reads them from its
context and leads RDBLearn. Where they are sparse and add little beyond activity, as on \textsc{rel-avito} user-clicks,
the order reverses. On \textsc{rel-trial} study-outcome $\rho_{\text{own}}$ is zero, and RT drops to chance although the
outcomes of related studies remain informative. RT-PluRel shares the model of RT but trails it on three of the four
tasks, so this gap comes from its synthetic prior. Only the closed KumoRFM-2 is strong in both regimes. These
observations lead to three requirements that existing work has not discussed.

\input{figures/challenges}
\input{figures/evidence}

\textbf{Requirement 1, longitudinal and cross-sectional learning.} A relational task is a panel over entities and time
(Figure~\ref{fig:challenges}a). \emph{Longitudinal} learning predicts an entity from its own past, including earlier
labels whose windows have ended, and \emph{cross-sectional} learning predicts it from a mapping between features and labels
fitted over other entities. Which dynamic dominates depends on the task. On \textsc{rel-f1} driver-top3, features such as
qualifying positions and constructor relate to the label, but the history labels carry much of the signal once they are available, since
a driver who often finishes in the top three is likely to do so again, and the label persistence is 22.6\%. On
\textsc{rel-avito} user-clicks the label persistence is 1.5\% and half of the queries have no history label, so features
across tables dominate. The balance also shifts within a task, since a new entity has no history, so a good predictor
must use both dynamics.

Existing methods each lean toward one dynamic. RDBLearn learns cross-sectionally. Its DFS aggregators treat the rows before the cutoff as an unordered set, which
loses the latest value and any recent change, and its features omit earlier task labels, which live outside the
database. As shown in Figure~\ref{fig:evidence:perturb}, RT learns mostly longitudinally. For the PluRel checkpoint with the RT decoder head, zeroing the feature cells of the context keeps almost all of the accuracy on \textsc{rel-f1} driver-dnf, whose contexts hold many earlier labels of the
driver, and zeroing the label cells brings it to chance. On \textsc{rel-event} user-ignore and \textsc{rel-avito}
user-clicks, whose contexts hold few earlier labels, zeroing the labels keeps most of the accuracy. The tested checkpoint thus draws much of its
accuracy from earlier labels whenever its context holds them and otherwise falls back on features, the regime in which
RDBLearn is stronger.

\textbf{Requirement 2, evidence selection within a bounded context.} An RFM sees a database only through a bounded
context, so the selection of that context decides which evidence reaches the model, and the mix of the two dynamics makes
this selection hard (Figure~\ref{fig:challenges}b). When the history of an entity is sparse, a sampler needs many hops to
collect earlier labels, and a predictor that relies on them works poorly. A context with features from related tables and
no history labels is often worse, since on many tasks a few history labels determine the label more directly than a
mapping from features, which is why RT-J~\citep{ranjan2026rtj} matches feature pipelines with far fewer labelled examples.
A sampler therefore has to balance both kinds of evidence within its budget.

\textsc{rel-trial} study-outcome is the extreme case. Its $\rho_{\text{own}}$ is zero and the outcomes of related studies
lie six foreign-key hops away, so 74.3\% of the sampled contexts hold no history label, against 25.0\% on \textsc{rel-f1}
driver-dnf, as shown in Figure~\ref{fig:evidence:context}. RT-J enlarges the context to 8{,}192 cells and retrieves rows with a
random walk, which narrows this gap only partly. Attention can only filter the cells it receives, and as the context
grows the few useful cells are buried among irrelevant ones, an effect also observed in large
language models with long contexts~\citep{liu2024lost}. A larger context is thus no substitute for a better selection, and we select evidence at
two levels, with a harness that rewrites the schema per task (Section~\ref{sec:harness}) and with a learned memory module over
the history inside the model (Section~\ref{sec:model}).

\textbf{Requirement 3, a pre-training prior with history-dependent labels.} RT-J adds real databases after PluRel and concludes that a strong RFM needs real
data~\citep{ranjan2026rtj}. With real data, however, it is hard to tell which property of the data a gain comes from, and the data cannot be
shaped to target the weaknesses of a model. Real relational databases are also scarce and often private. We keep pre-training fully synthetic and ask what PluRel misses. A model
trained on data drawn from a prior approximates the posterior predictive distribution under that prior~\citep{muller2022pfn}.
It can therefore use a dependence between context and label at test time only if the prior contains that dependence. Previous synthetic relational data~\citep{kothapalli2026plurel} is built from the top down, with each cell a function of the latent inputs
of its row and of its parent rows. Information flows only from an entity to its events, so a label depends on the attributes of the entity alone. In real
tasks the label often depends on history. We call a label \emph{history-dependent} if it depends on the past events of
an entity or of its related entities. The strongest evidence for user churn, for example, is a windowed count or the latest
event of the user. RT attends to the child rows that hold this history, yet PluRel puts almost no prior mass on such labels.

Table~\ref{tab:operators} tests this prediction on synthetic worlds where each label is one aggregate plus noise. On
the static control, every method nearly reaches the oracle. On history-dependent labels, PluRel falls far below it, and
\openrfm{}, pre-trained on our prior, closes most of the gap on the window count and the
order of events. A synthetic prior can thus keep the broad support of PluRel and add history-dependent labels, as
Section~\ref{sec:data} formalises.

\input{tables/operators}

\subsubsection{Model}
\label{sec:model}

\openrfm{} predicts with a learned relational adapter and a frozen tabular foundation model
(\hyperref[fig:pipeline]{Figure~\ref*{fig:pipeline}c--d}). For each task row, a sampler builds a context from the database, the adapter
$\phi_\theta$ encodes this context into one feature row, and a memory module adds columns from the history that the
context cannot hold. The frozen tabular foundation model $g$ then predicts each query from labelled support rows.
Tabular foundation models already learn well from labelled rows and improve every few
months~\citep{hollmann2025tabpfn,qu2025tabicl,jager2026tabpfn35,nvidia2026kumotabular}, so a small adapter that writes
rows for them is cheap to train and inherits each new backbone. By default the sampler walks the original schema
breadth-first from the task row (Appendix~\ref{app:default}). Algorithm~\ref{alg:openrfm} in Appendix~\ref{app:impl} summarises prediction and
pre-training.

\textbf{Adapter.} The adapter builds on RT~\citep{ranjan2025relational} and keeps its cell tokens, attention masks and
masked-cell objective (\hyperref[fig:pipeline]{Figure~\ref*{fig:pipeline}c}). For a task row $(e,t)$ the sampler (Appendix~\ref{app:impl}) returns
a context $C_{e,t}$ of at most $L$ cells from $\mathcal{D}_{<t}$. A cell with row $r$, column $c$ and value $v$ of type
$d(v)$ becomes the token
\begin{equation}
\label{eq:cell_embedding}
\vz_{r,c} = \mW_{d(v)}\,\mathrm{enc}_{d(v)}(v) + \mW_{\mathrm{s}}\,E_{\mathrm{text}}\big(c,T(r)\big) + \ve_{\mathrm{role}(r,c)} .
\end{equation}
Here $\mathrm{enc}_d$ encodes a value of type $d$, with numbers standardised per column, timestamps mapped through radial
basis functions and text passed through a frozen text encoder. $E_{\mathrm{text}}$ is a frozen embedding of the column
and table names, $\mW_d$ and $\mW_{\mathrm{s}}$ are learned projections, and $\ve_{\mathrm{role}}$ embeds the role of the cell, such
as membership of the focal entity. A hidden cell uses a learned vector $\vm_{d(v)}$ in place of its value term, so its
token keeps the column and role and hides the value. Because tokens are unordered, the output is invariant to the sampling order. The task row of $(e,t)$ holds the cutoff and the hidden \emph{target cell}.

For token $i$, let $\mathrm{Col}(i)$ and $\mathrm{Row}(i)$ be its column and row, $\mathrm{Out}(i)$ the rows that
$\mathrm{Row}(i)$ references, and $F$ the rows of $C_{e,t}$ that carry the key of $e$. An attention layer computes
$\mathrm{Attn}(\mQ,\mK,\mV;\mM)=\mathrm{softmax}\big(\mQ\mK^\top/\sqrt{d_k}+\log\mM\big)\mV$ with a binary mask $\mM$. A
block applies the three relational layers of RT,
\begin{equation}
\label{eq:masks}
\begin{aligned}
\mM_{\mathrm{col}}[i,j] &= \mathbf{1}\{\mathrm{Col}(j)=\mathrm{Col}(i)\},\\
\mM_{\mathrm{feat}}[i,j] &= \mathbf{1}\{\mathrm{Row}(j)=\mathrm{Row}(i)\lor\mathrm{Row}(j)\in\mathrm{Out}(i)\},\\
\mM_{\mathrm{nbr}}[i,j] &= \mathbf{1}\{\mathrm{Row}(i)\in\mathrm{Out}(j)\},
\end{aligned}
\end{equation}
each as $\mZ\gets\mZ+\mathrm{Norm}\big(\mathrm{MHA}(\mZ;\mM)\big)$, then an entity layer and an MLP. A cell with no allowed key, such as a cell of a row without sampled children in the neighbour layer, gets a zero attention output in that layer. The column layer
models the distribution of a column, the feature layer joins a row with its parents, and the neighbour layer aggregates
its children.

\textbf{Entity layer.} RT ends each block with full attention. We replace it with an entity layer over the rows of $e$
in all tables. A softmax returns a weighted mean and therefore cannot count. A user with three reviews and a user with
three hundred reviews of the same kind look alike, yet real tasks need such counts, as Requirement~3 states. The entity layer
pools each row $\rho\in F$ into one token $\vu_\rho$, the mean of its cell states, and attends to these tokens with mean heads $\mathcal{H}_{\mathrm{m}}$
and sum heads $\mathcal{H}_{\mathrm{s}}$,
\begin{equation}
\label{eq:entity}
\begin{aligned}
\vo_i &= \sum_{\rho\in F}\operatorname{softmax}_{\rho}\Big(\frac{\vq_i\cdot\vk_\rho}{\sqrt{d_k}}\Big)\vv_\rho
&&\text{mean head},\\
\vo_i &= \operatorname{slog}\Big(\sum_{\rho\in F}\operatorname{sigmoid}\Big(\frac{\vq_i\cdot\vk_\rho}{\sqrt{d_k}}\Big)\vv_\rho\Big)
&&\text{sum head},
\end{aligned}
\end{equation}
where $\vq_i$ is a projection of the token state $\vz_i$, $\vk_\rho$ and $\vv_\rho$ are projections of the row token
$\vu_\rho$, each head has its own projections, and $\operatorname{slog}(x)=\operatorname{sign}(x)\log(1+|x|)$. A sum head
weights each row in $(0,1)$ independently, so its output grows with the number of selected rows and can encode the count of a window. The mean heads keep the latest or the typical value. The outputs of all heads are concatenated and added to the token state,
$\vz_i\gets\vz_i+\mW_{\mathrm{o}}[\vo_i]_{\mathrm{heads}}$, with $\mW_{\mathrm{o}}$ initialised at zero.

\textbf{Memory module.} The context holds only the most recent part of a long history. The memory module $r_\psi$
gives the model access to a longer stretch of it. For a task row $(e,t)$ the memory $T_{e,t}$ is a set of tokens of two
kinds. The first are the token states $\vz_i$ of the context before the first block and after each block, so the later
ones already carry the outputs of the entity layer. The second are history tokens, one per cell of $H_{e,t}$, the newest
$N$ rows before $t$ that $e$ reaches through foreign keys. A history token is the frozen cell embedding of
\eqref{eq:cell_embedding} plus embeddings of the age of its row and of its foreign-key path. A set of $Q$ learned queries
$\vq_j$ reads the memory with one cross-attention layer,
\begin{equation}
\label{eq:read}
s_{j,m}=\frac{\vq_j\cdot\vk_m}{\sqrt{d_k}},\qquad
A_j=\log\!\sum_{m\in T_{e,t}}e^{s_{j,m}},\qquad
\bar v_j=\sum_{m\in T_{e,t}}\operatorname{softmax}_{m}\big(s_{j,m}\big)\,v_{j,m},
\end{equation}
where $\vk_m=\mW_K\vz_m$ and each query has its own scalar value $v_{j,m}=\vw_j^\top\vz_m$. Each query exports two
numbers. The softmax read $\bar v_j$ returns a latest or typical value, as the mean heads of the entity layer do. The
log-partition $A_j\approx s^\star_j+\log n$, with $s^\star_j$ the score of the selected tokens, grows with the number $n$ of tokens that
the query selects and so encodes a count, as the sum head does. The $Q$ queries thus summarise the long history in $2Q$
columns. They follow the same idea as the ``thinking'' rows that TabPFN-2.5 adds to its input and learns during
pre-training to give the model extra computation~\citep{grinsztajn2025tabpfnv2_5}. The context therefore keeps the recent
rows at full cell resolution, and the memory adds a coarser view of the longer history.

\textbf{Readout.} A readout replaces the decoder head of RT and writes one input row of the TFM backbone per task row from the
adapter and the memory module,
\begin{equation}
\label{eq:readout}
x_{e,t} = \Big[\phi_\theta(C_{e,t}),\ r_\psi(C_{e,t},H_{e,t})\Big],\qquad
\phi_\theta(C_{e,t}) = \Big[\vz^{(B/2)}_{\star},\ \vz^{(B)}_{\star}\Big],
\end{equation}
where $\vz^{(B/2)}_{\star}$ and $\vz^{(B)}_{\star}$ are the states of the target cell after the middle and the final of
the $B$ blocks and $r_\psi(C_{e,t},H_{e,t})$ are the memory-module columns. At the middle block the target cell still holds
cell-local evidence, such as a count or the latest value, that later blocks mix with the rest of the context, and the
final block holds the most contextualised state.
The TFM backbone supplies a second context. Graph sampling gives a task row only the rows its entity reaches
through foreign keys, so entities that share no key never see each other, and a mapping between features and labels
would have to be learned from the few labels inside one subgraph. The backbone attends across the rows of a batch, where
each query row reads the labelled support rows of other entities, and fits this mapping in context, which is the
cross-sectional learning of Requirement~1. Each prediction therefore uses two explicit contexts, a relational context
along foreign keys that the adapter reads and a batch context across entities that the backbone reads.

\subsubsection{Synthetic prior and pre-training}
\label{sec:data}

\textbf{Prior fitting.} A prior-data fitted network~\citep{muller2022pfn} learns Bayesian prediction from a prior over
data-generating processes. Each episode draws a process ${\zeta}\sim\Pi$ and a dataset from it, splits the dataset into
labelled examples $\mathcal{S}$ and a query $x$, and trains the network to predict the label of $x$. Minimising the
expected log loss approximates the posterior predictive distribution
\begin{equation}
\label{eq:ppd}
p_\Pi(y\mid x,\mathcal{S}) = \int p(y\mid x,{\zeta})\,p({\zeta}\mid\mathcal{S})\,d{\zeta} .
\end{equation}
TabPFN and TabICL follow this recipe with priors over single tables.

\textbf{Prior and posterior of a relational adapter.} In a single table, examples are independent rows with labels. In a
relational database, all examples live in one database, so an episode starts by drawing a whole database, which we call
a world. A process ${\zeta}=(\xi,\tau)$ has world parameters $\xi$ and a label function $\tau$. An episode draws a world
$W\sim p(W\mid\xi)$ of tables, keys and event streams, then draws pairs $(e,t)$ of an entity and a cutoff. Each pair is
an example with input context $C_{e,t}$ sampled from $W_{<t}$ and label $y=\tau(W,e,t)$, which may depend on events after
$t$, and the other pairs of the same world form the labelled set $\mathcal{S}$. In \eqref{eq:ppd}, the likelihood
$p(y\mid C_{e,t},{\zeta})$ predicts an entity from its own past through longitudinal learning, and the posterior
$p({\zeta}\mid\mathcal{S})$ infers the task from other entities through cross-sectional learning, giving the two dynamics of
Requirement~1. Because the frozen tabular backbone of \openrfm{} sees only one row per entity written by the adapter, the
adapter must compute any statistic of an entity's history, such as a count in a recent window, and the prior decides
which statistics it learns to compute.

\input{figures/prior}

\textbf{Worlds.} Figure~\ref{fig:prior} shows the schema of a world. Every
entity references a hub, and every hub references an ancestor. Entities with the same hub $p(e)$
are peers. In PluRel~\citep{kothapalli2026plurel} each event picks its parent independently of time and of earlier
events, so the events of an entity do not depend on its earlier events. We let the events of an entity depend on
its own past and on its hub through the Hawkes process~\citep{hawkes1971spectra}, the simplest point process whose rate
depends on past events. The events of $e$ in relation $r$ have intensity
\begin{align}
\label{eq:intensity}
\nu_{e,r}(t) &= (1-b)\,\mu_{e,r}(t) + \textstyle\sum_{t_j<t}\frac{b}{\omega}\,e^{-(t-t_j)/\omega},\\
\mu_{e,r}(t) &= \underbrace{\tfrac{\bar n}{T}\,c_r\,e^{\gamma\,s_e(t)}}_{\text{level and drift}}
\cdot\underbrace{a_e(t)}_{\text{activity}}
\cdot\underbrace{\big(1+\beta\sin(2\pi t/P+\varphi_e)\big)}_{\text{season}},\nonumber\\
s_e(t) &= \lambda_e+\big(\alpha\,d_e+\alpha'\,d_{p(e)}\big)\big(t/T-1/2\big),\nonumber
\end{align}
where $T$ is the time span, $\bar n$ the mean event count per entity, $c_r$ a scale per relation and $\gamma$, $\alpha$
and $\alpha'$ constants, and $\varphi_e\sim\mathcal{U}[0,2\pi)$ is a random seasonal phase. The rate level $\lambda_e$ is correlated with the level $u_e$ that sets the values of the
events, $u_e$ is partly shared with the hub, and the hub drift $d_{p(e)}$ moves peers together. The activity $a_e(t)$,
a two-state Markov chain~\citep{fischer1993mmpp}, lets entities churn and return. Each event has on average $b<1$
offspring after delays of mean $\omega$, which creates bursts, and we sample them through the cluster
representation~\citep{hawkes1974cluster}. Each event becomes a row of $r$ that references $e$, with a numeric and a
categorical mark that follow $u_e$ plus noise. Appendix~\ref{app:worldgen} gives the full process with all constants,
and Appendix~\ref{app:plurel} compares a world with real databases and with PluRel. Other families use generators with
their own schemas.

\input{tables/targets_inline}

\textbf{Label functions.} A label function maps an entity $e$ and a cutoff $t$ to a label $y_e(t)$, which the model
predicts from the rows before $t$. Table~\ref{tab:targets} lists the label families and their shares of episodes. A
future-window target with horizon $H$ is
\begin{equation}
\label{eq:future_target}
y_e(t) = f\big(\{(t_j, m_j{,c_j}) \in r_e : t < t_j \le t+H\}\big),
\end{equation}
where $r_e$ holds the events of $e$ in relation $r$ with numeric marks $m_j$ and categories $c_j$, and $f$ is one of a fixed set of aggregates. The label of cutoff $t$
enters the database at $t+H$, so earlier labels become history for later queries, as in real tasks
(Section~\ref{sec:prelim:rfm}). Such a label is history-dependent, and the past events and history labels of the entity and its
peers all predict it. Other families each ask for one computation, such as an exact count in a window or a choice among
items shown together.

\textbf{Comparison with PluRel~\citep{kothapalli2026plurel}.} As shown in Table~\ref{tab:operators}, a model pre-trained on PluRel comes within one point of the oracle on
static attributes but loses 12 to 36 points on the three history-dependent labels, because the Bayes predictor of PluRel needs few of the
statistics of real tasks, for three reasons. (1) PluRel computes each cell from its row and its parent rows, so its labels
never need the child rows of an entity, where real labels such as churn find their evidence. (2) PluRel attaches child
rows to parents independently, so the events of an entity are a random subset of a table and have no timeline of their
own. (3) Its random MLPs produce smooth regression targets, whereas real targets aggregate events into counts, sums and
waiting times with many zeros and heavy tails. Our labels aggregate the events of an entity and its peers, generated as
a process in time, so all of these statistics become informative.

\textbf{Pre-training objective.} The adapter is trained with two losses. The masked-cell loss $\mathcal{L}_{\mathrm{cell}}$ of RT
hides random cells of each context and predicts their values, which keeps the adapter informative about raw values.
It is computed at every update. The in-context loss $\mathcal{L}_{\mathrm{ICL}}$ trains the adapter rows for the frozen tabular model. Each episode draws a synthetic task, a support set $\mathcal{S}$ and queries $\mathcal{Q}$, and $g$
predicts the queries from the adapter rows (Algorithm~\ref{alg:openrfm}). The in-context loss is
\begin{equation}
\label{eq:icl_loss}
\mathcal{L}_{\mathrm{ICL}}(\theta) = \frac{1}{|\mathcal{Q}|}\sum_{(e,t,y)\in\mathcal{Q}}
\ell\Big(g\big({\phi_\theta(C_{e,t})}\,\big|\,\{({\phi_\theta(C_{e_k,t_k})},y_k)\}_{(e_k,t_k,y_k)\in\mathcal{S}}\big),\ y\Big),
\end{equation}
where $\ell$ is cross-entropy for classification and a Huber loss on support-scaled labels for regression. The gradient
passes through the frozen $g$ into $\theta$. On the updates that include $\mathcal{L}_{\mathrm{ICL}}$, both gradients are rescaled to equal norm, and the in-context gradient is multiplied by a weight $w_t$ that ramps up from zero early in training, before they are
added (Algorithm~\ref{alg:openrfm}). As at inference, support labels end before the query cutoff.

\textbf{Two training stages.} Stage~1 trains the adapter $\phi_\theta$ with both losses. Each episode draws a context budget
$L$ and a support size $K$, and large support sets appear only with short contexts. Stage~2 freezes $\phi_\theta$ and $g$
and trains the memory module $r_\psi$. It uses the in-context loss with the full row
$x_{e,t}=\big[\phi_\theta(C_{e,t}),\,r_\psi(C_{e,t},H_{e,t})\big]$ in place of the adapter row,
\begin{equation}
\label{eq:mem_loss}
\mathcal{L}_{\mathrm{mem}}(\psi) = \frac{1}{|\mathcal{Q}|}\sum_{(e,t,y)\in\mathcal{Q}}
\ell\Big(g\big(x_{e,t}\,\big|\,\{(x_{e_k,t_k},y_k)\}_{(e_k,t_k,y_k)\in\mathcal{S}}\big),\ y\Big),
\end{equation}
so its gradient reaches only $\psi$. Stage~2 draws worlds whose evidence lies in long, irregular histories beyond
the context window. A sampler picks the next family in proportion to the loss that the memory module can still remove,
measured against an oracle that reads the exact history. Each world serves its own task and label programs over its
past rows or its future window, some of them SQL written by an LLM. Appendix~\ref{app:memworlds} describes the families,
the sampler, the programs and all values.

\subsection{Improving \openrfm{} with harness design}
\label{sec:harness}

For LLM agents, the harness around a model decides what enters its context, and the same weights can perform very
differently under different harnesses~\citep{yang2024sweagent}. An RFM has the same dependence. With the default sampler and the original schema, the adapter computes only over the cells that the
sampler delivers within the budget $L$, and the backbone compares a query only with the support set it receives. On
\textsc{rel-trial} study-outcome, the best evidence for a study is how earlier studies with the same sponsor, condition
or intervention turn out. These studies lie six foreign-key hops away through hub tables, so the default breadth-first
walk spends its budget on intermediate rows, and most contexts contain no history label, as shown in Figures~\ref{fig:challenges}b
and~\ref{fig:evidence:context}. Enlarging the budget is an inefficient fix. Most of the added rows are unrelated to the study, and the score improves
little. A better context at the same budget comes from two levels. The first level is the batch axis. The
frozen TFM backbone compares each query with $K$ labelled support rows, so the labels of other task rows enter the
prediction even when the cells sampled for the query contain none. On this task, \openrfm{} with the default harness,
which adds this axis and our synthetic prior to the RT adapter, reaches 72.75 AUROC against 53.8 for RT-PluRel
(Table~\ref{tab:prelim_study}). The second level is the density of the relational signal within each context. A
rewrite that links each study to its peers puts known peer outcomes into almost every context, and the share of label
cells in a context of $L=1{,}024$ cells rises from 0.9\% to 20.5\%. With this rewrite, the searched harness raises
the score to 74.34
(Table~\ref{tab:main_relbench_v1_cls}). The useful rewrite depends on the schema and the task, so an LLM agent configures the harness per task while
the weights stay fixed.

\textbf{Elements of an RFM harness.} A harness configuration $h=(\pi,\sigma,\kappa)$ has three parts. (1) The \emph{schema program} $\pi$ rewrites the database for the task. It adds generated tables that expose distant or aggregated evidence, together with feature recipes such as generated columns and, for regression, an invertible target transform applied to the labels before prediction and inverted afterwards. Every generated row uses only rows before its cutoff and history labels. (2) The \emph{sampler} $\sigma$ builds the context of each task row within a budget of $L$ cells. The sampling algorithm is open-ended, and the default walks the database breadth-first from the task row. (3) The \emph{support set} $\kappa$ gives the number of support rows $K$, their order and the number of support draws $E$. The selection algorithm is also open-ended, for example random selection or graph-based active selection~\citep{zhang2026rdblearn}. Appendix~\ref{app:protocol} gives the values searched in our experiments.

\textbf{The proposal loop.} The LLM agent proposes a harness configuration, receives a reward and revises the proposal.
The reward takes two forms. (1)~A performance score evaluates the frozen model on training rows, on validation rows or on a
mix of the two. (2)~A model-agnostic score measures the evidence that the harness delivers, such as the share of label cells or of peer labels in the sampled contexts. Appendix~\ref{app:protocol} describes the
programs and the rewards of our experiments.

\textbf{Relation to RelAgent.} RelAgent~\citep{huang2026relagent} writes features from the semantics of the schema
and fits a gradient-boosted model on them. A harness for an RFM works on both the feature side and the relation side. On
the feature side, the schema program also writes RelAgent-style features, such as windowed counts, as generated
columns, which give exact values at a finer resolution than what the model can infer from the sampled rows. On
the relation side, signals beyond semantics can guide the rewrite, such as label homophily along a
relation~\citep{chen2026relatron} and the own-label density $\rho_{\text{own}}$. Figure~\ref{fig:study_demo} compares
both programs on \textsc{rel-trial} study-outcome. Our program links each study to earlier studies that share a sponsor,
condition or intervention, so their known outcomes enter the context.

%% file: figures/challenges.tex
\begin{figure}[t]
\centering
\includegraphics[width=\linewidth]{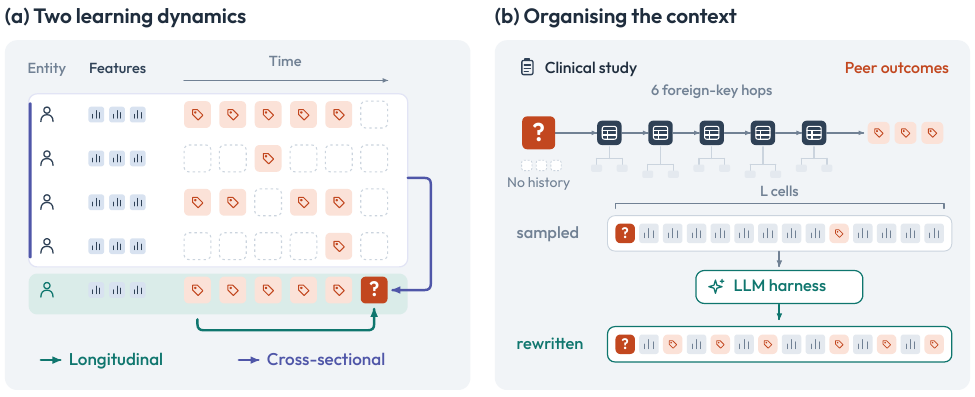}
\caption{\textbf{Two requirements of relational prediction.} (a) A task is a panel over entities (rows) and time
(columns), with features (bars), history labels (coral tags) and the hidden query label, drawn as a question mark. Longitudinal
learning reads the entity's own earlier labels, and cross-sectional learning fits a mapping from features to labels over other
entities. (b) On \textsc{rel-trial} study-outcome the peer outcomes lie six foreign-key hops away, so a sampled context holds
few of them until the LLM harness rewrites the schema.}
\label{fig:challenges}
\end{figure}

%% file: figures/evidence.tex
\begin{figure}[t]
\centering
\begin{minipage}[t]{.45\textwidth}
\vspace{0pt}
\vbox to \dimexpr.470\linewidth+15pt\relax{\vfil\hbox to \linewidth{\hfil\footnotesize\setlength{\tabcolsep}{2.4pt}\renewcommand{\arraystretch}{1.35}%
\begin{tabular}{@{}lcccccc@{}}
\toprule
Task & RDBL. & RT & \makecell{RT-\\PluRel} & \makecell{Kumo\\RFM-2} & $\lambda$ & $\rho_{\text{own}}$ \\
\midrule
\textsc{f1} dnf & 70.9 & 81.2 & 76.7 & 84.6 & 4.7 & .94 \\
\textsc{f1} top3 & 79.7 & 89.3 & 82.6 & 92.2 & 22.6 & .95 \\
\textsc{avito} clicks & 69.0 & 59.5 & 45.9 & 67.4 & 1.5 & .48 \\
\textsc{trial} outcome & 71.6 & 51.8 & 53.8 & 72.0 & 0.0 & .00 \\
\bottomrule
\end{tabular}\hfil}\vfil}
\captionof{table}{The two routes on four RelBench tasks (test AUROC). Scores are from \citet{huang2026relagent}, and
RT-PluRel from \citet{kothapalli2026plurel}. $\lambda$ is in percent.}
\label{tab:prelim_study}
\end{minipage}\hfill
\begin{minipage}[t]{.53\textwidth}
\vspace{0pt}
\captionsetup{type=figure}
\centering
\subcaptionbox{Perturbation\label{fig:evidence:perturb}}{\includegraphics[width=.3585\linewidth]{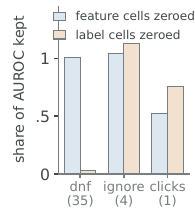}}\hspace{.06\linewidth}
\subcaptionbox{Labels in context\label{fig:evidence:context}}{\includegraphics[width=.3585\linewidth]{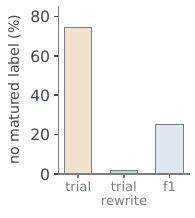}}
\caption{\textbf{Evidence for the requirements.} (a) AUROC above chance kept by PluRel with the RT decoder head when
context cells are zeroed, with own earlier labels per context in brackets. (b) Contexts without a history label.}
\label{fig:evidence}
\end{minipage}
\end{figure}

%% file: tables/operators.tex
\begin{table}[t]
\centering
\caption{History-dependent labels on synthetic worlds, AUROC $\times 100$. Each label is one aggregate over past events plus noise, except the control, a function of static attributes. The oracle
gives the backbone the true aggregate value as a feature and serves as a reference for what a context can reach.}
\label{tab:operators}
\small
\begin{tabular}{@{}lcccc@{}}
\toprule
Aggregate & Oracle & RDBLearn & PluRel & \openrfm{} \\
\midrule
Static attribute (control) & 90.2 & 89.8$\pm$0.6 & 89.4$\pm$0.8 & 88.2$\pm$1.1 \\
\midrule
Count in a recent window & 96.7 & 71.7$\pm$0.8 & 61.0$\pm$1.6 & \textbf{86.4}$\pm$0.8 \\
Latest value & 93.3 & 54.3$\pm$0.7 & 63.5$\pm$1.6 & \textbf{68.7}$\pm$0.9 \\
Order of events & 97.4 & 50.9$\pm$0.7 & 85.5$\pm$0.4 & \textbf{95.7}$\pm$0.2 \\
\bottomrule
\end{tabular}
\end{table}

%% file: figures/prior.tex
\begin{figure}[t]
\centering
\begin{minipage}[t]{0.33\linewidth}
\centering
\includegraphics[height=3.9cm]{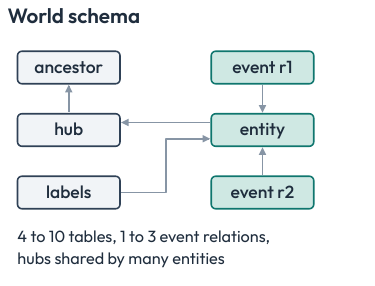}
\caption{A synthetic world. Entities reference shared hubs, hubs reference ancestors, and event and label rows reference entities.}
\label{fig:prior}
\end{minipage}\hfill
\begin{minipage}[t]{0.63\linewidth}
\centering
\includegraphics[height=3.9cm]{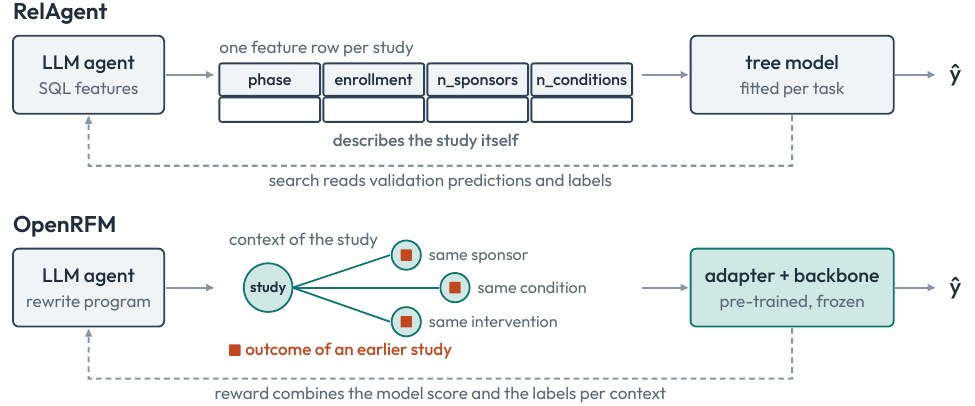}
\caption{The two harnesses on \textsc{rel-trial} study-outcome. RelAgent fits a model on SQL features of each study.
\openrfm{} rewrites the schema so that the context holds earlier related studies and their outcomes.}
\label{fig:study_demo}
\end{minipage}
\end{figure}

%% file: tables/targets_inline.tex
\begin{table}[t]
\centering
\caption{Label families of the pre-training corpus and their share of training episodes. Table~\ref{tab:targets_full} in
Appendix~\ref{app:prior} splits the five groups into ten more fine-grained families.}
\label{tab:targets}
\small
\renewcommand{\arraystretch}{1.2}
\begin{tabularx}{\linewidth}{@{}>{\raggedright\arraybackslash}p{2.7cm}>{\raggedright\arraybackslash}X>{\raggedright\arraybackslash}p{4.3cm}r@{}}
\toprule
Family & Label of entity $e$ at cutoff $t$ & What the model must use & Share \\
\midrule
Future activity & whether, how often, how much or when $e$ acts in $(t,t+H]$, or a rate change or burst in a window & counts, recency, bursts, past labels & 23.5\% \\
History aggregates & a count, sum, mean, quantile, maximum or latest value over past events of $e$ and its peers, an SQL query over them, or a value looked up in a parent & own and peer aggregates, lookups & 26.7\% \\
Changing outcomes & an outcome of $e$ that drifts, reverts, switches or repeats over time & past labels, trends & 13.0\% \\
Autocomplete & a hidden value of an existing row, also for rows whose parents change & row, parent and peer values & 21.2\% \\
Rows and choices & a nonlinear function of row values, a mapping given only by the labelled examples, or the item chosen from a group shown together & row functions, item comparison & 15.6\% \\
\bottomrule
\end{tabularx}
\end{table}

%% file: sections/experiments.tex
\section{Experiments}
\label{sec:exp}

We evaluate the effectiveness of \openrfm{} by asking three research questions. RQ1 asks how \openrfm{} compares with published methods. RQ2 asks which parts of pre-training, the model and the harness matter. RQ3 asks how the adapter scales with its size.

\subsection{Setup}
\label{sec:exp:setup}

\textbf{Benchmark.} We evaluate on the 32 tasks adopted by \citet{huang2026relagent}, comprising 21 classification tasks scored by AUROC and 11 regression tasks scored by MAE. Twenty-one tasks come from RelBench-V1~\citep{robinson2024relbench}, six come from RelBench-V2 and five come from 4DBInfer~\citep{wang2024fourdbinfer}. We use the temporal splits of each source benchmark. Baseline results are taken from \citet{huang2026relagent}.

\textbf{Settings.} We consider three settings. R1 runs the pre-trained model with the default harness, which keeps the original schema and the default breadth-first sampler. R2 runs the same model with an LLM-searched harness, in which the LLM performs open-ended schema generation, including feature transformation and feature aggregation. R3 adds to R2 a zero-initialised residual head, a small MLP whose output is added to that of the frozen TFM backbone. For each task the head is trained on the training rows with the R2 harness fixed and all other weights frozen, using cross-entropy for classification and a pinball loss for regression. As the frozen TFM backbone we consider Kumo-Tabular~\citep{nvidia2026kumotabular}, TabPFN-3.5~\citep{jager2026tabpfn35} and TabICL v2~\citep{qu2026tabiclv2}, and compare them in Section~\ref{sec:exp:ablation}. Kumo-Tabular performs best, and we adopt its small model as the main TFM backbone. During pre-training, the in-context loss of the adapter uses frozen TabICL v2, and the memory module is trained with frozen Kumo-Tabular small. The harness search uses GPT-6.1 through the Codex CLI as the LLM, and DuckDB executes the schema programs. Pre-training runs on a single B200 GPU, and evaluation runs on a mix of B200, GH200, H200 and RTX A6000 GPUs.

\subsection{Main results (RQ1)}
\label{sec:exp:main}

\input{tables/main_relbench_v1_cls}
\input{tables/main_relbench_v1_reg}
\input{tables/main_relbench_v2}
\textbf{Observation 1. With the harness, \openrfm{} ranks best overall without training on the target databases.} For the
mean rank over all 32 tasks, we rank R1, R2 and R3 together with all published methods on each task
(Tables~\ref{tab:main_relbench_v1_cls}--\ref{tab:main_relbench_v2_reg}), where tied methods share the better rank. R1 and R2 share the same weights, so the gain of R2 over R1 comes from the harness alone, and R2
scores better than R1 on 27 of the 32 tasks. Without the harness, R1 is on par with KumoRFM-2, the strongest RFM
baseline. The residual head of R3, the only part trained on the target tasks, changes the scores only slightly. Over the 32 tasks, the joint mean ranks order the methods as R3 (2.500), R2 (3.000), RelAgent (3.625), KumoRFM-2
(5.344), R1 (5.469), GraphSAGE (8.156) and KumoRFM (9.688), and R3, R2 and RelAgent rank first on 14, 8 and 9 of the 32 tasks, respectively.

\textbf{Observation 2. \openrfm{} leads on regression and is close to the best methods on classification.} On the 11
regression tasks R2 ranks on average at least two places ahead of RelAgent and KumoRFM-2, while on classification RelAgent
keeps a slight lead. Classification
scores lie close together across methods, whereas regression
errors differ widely.

\textbf{Observation 3. The remaining gaps lie on activity tasks.} R2 ranks fifth or lower on
\textsc{rel-amazon} user-churn and item-churn, \textsc{rel-avito} user-clicks and user-visits and \textsc{rel-stack}
user-engagement, and trails the best method by a few AUROC points on each. Supervised methods fit the millions of training rows of \textsc{rel-amazon} and \textsc{rel-stack}, while the
frozen TFM backbone sees at most tens of thousands of support rows. Larger or better-selected support sets, or task-specific training, may close these
gaps.

\subsection{Ablation study (RQ2)}
\label{sec:exp:ablation}
\suppressfloats[t]

For the ablation study we select 9 representative tasks as listed in Appendix~\ref{app:protocol}. For each variant, we insert its result on each task alone into the corresponding result table of Section~\ref{sec:exp:main}, rank it among the published methods, and compare the variants by their average rank over these tasks.

\begin{table}[t]
\centering
\footnotesize
\input{tables/ablation_all}
\end{table}

\textbf{Pre-training.} We first consider the effect of the pre-training strategy. As shown in Table~\ref{tab:ablation}a, with the
same frozen TFM backbone, the adapter pre-trained on our prior ranks ahead of one pre-trained on
PluRel~\citep{kothapalli2026plurel}, 5.667 against 6.444, and PluRel with its original RT head ranks far below both.
Removing the autocomplete or the future-activity families worsens the rank, so the families of our
prior contribute. The prior is necessary, yet a transformer with a bounded context can apply it only to the history that
fits in its context. Adding the memory module, which brings far longer histories into view, lifts the rank from 5.667 to
3.111, with the largest gains on \textsc{rel-avito} user-clicks and \textsc{rel-trial} study-outcome. The module also reads the context states again, so this gain combines the longer history with a second readout of the context.

\textbf{Architecture.} We next ablate the parts of the adapter output and the choice of the TFM backbone, as
shown in Table~\ref{tab:ablation}b. The adapter output joins the state of the target cell after the middle block, its state after
the final block and the memory module (Equation~\ref{eq:readout}). Removing the memory module or the middle-block state raises the rank from 3.111 to 5.667, and removing the final-block state raises it only to 3.444. Removing both target-cell states gives 4.111, and the full row ranks best among these variants. Because the adapter writes plain feature columns, the TFM backbone can be replaced at inference time without retraining
the adapter, which is pre-trained with TabICL v2. Among the frozen TFM backbones, Kumo-Tabular small ranks best. TabICL v2 and TabPFN-3.5 run here with at most 16{,}384 and
8{,}192 support rows, fewer than the largest support sets that Kumo-Tabular uses. We see two possible reasons for the
advantage of Kumo-Tabular. First, its pre-training data may resemble the relational prior of RDB-PFN~\citep{wang2026relational},
which would suit the aggregated features that the adapter writes. Second, it accepts larger support sets, which help on
tasks with many training rows. At the same support size Kumo-Tabular small still ranks ahead, so the support size alone
does not explain the gap.

\textbf{Harness.} Turning to the harness, we ablate two of its directions in R2, as shown in Table~\ref{tab:ablation}c,
(1) the schema program and (2) the selection of the support set. For the first direction, we observe that the table
rewrite, which makes the history signal in the context denser as described in Section~\ref{sec:harness}, matters most,
and reverting it raises the mean rank from 2.469 to 4.938, where R2 is re-scored with random support order so that both rows
use the same support order. This variant ranks below R1 because it keeps the support size, the number of draws and the other hyperparameters that the search tuned for the rewritten schema. The hyperparameters are therefore also an important part of the harness, and they must be chosen together with the schema program. The program also transforms the regression target, for
example with $\operatorname{log1p}$ or $\operatorname{sqrt}$, and reverting this transform raises the rank to 2.656. For
the second direction, we observe that reverting the tuned support size and number of draws to their R1 values raises
the rank to 3.031, a smaller loss than reverting the table rewrite.

\subsection{Scaling (RQ3)}
\label{sec:exp:scaling}

\begin{wrapfigure}{r}{0.47\linewidth}
\vspace{-30pt}
\input{figures/scaling}
\vspace{-10pt}
\end{wrapfigure}
As shown in Figure~\ref{fig:scaling}, a larger adapter reaches a lower held-out masked-cell loss. Across five adapters from
0.96M to 26.0M parameters, the masked-cell loss falls and flattens beyond 10M. The size of the adapter is also bounded by the
column budget of the TFM. The adapter output enters the TFM as
feature columns, and Kumo-Tabular is designed for at most 500 of them. With a hidden width of 128, the two target-cell states and
the memory module write 384 columns, and the next width of 192 would exceed the default cap of 500 input columns of Kumo-Tabular. We therefore use the 3.5M adapter.
With the 0.28M memory module and the frozen Kumo-Tabular small model, which has 27.5M parameters for classification and
28.5M for regression, \openrfm{} has about 32M parameters in total.

%% file: tables/main_relbench_v1_cls.tex
\begin{table}[t]
\centering
\caption{Test AUROC ($\uparrow$) on the RelBench-V1 classification tasks. Shading marks the top three per column, bold the best. Ranks follow Section~\ref{sec:exp:main}, with all rows, including the three \openrfm{} rows, ranked together. Later result tables use the same practice.}
\label{tab:main_relbench_v1_cls}
\small
\setlength{\tabcolsep}{3pt}
\resizebox{\linewidth}{!}{%
\begin{tabular}{@{}lrrrrrrrrrrrrrr@{}}
\toprule
Method & \multicolumn{2}{c}{f1} & \multicolumn{2}{c}{avito} & \multicolumn{2}{c}{event} & \multicolumn{1}{c}{trial} & \multicolumn{2}{c}{amazon} & \multicolumn{2}{c}{stack} & \multicolumn{1}{c}{hm} & Avg & Rank \\
\cmidrule(lr){2-3}\cmidrule(lr){4-5}\cmidrule(lr){6-7}\cmidrule(lr){8-8}\cmidrule(lr){9-10}\cmidrule(lr){11-12}\cmidrule(lr){13-13}
 & dnf & top3 & click & visit & repeat & ignore & out & user & item & eng & badge & churn & $\uparrow$ & $\downarrow$ \\
\midrule
\multicolumn{15}{@{}l}{\textit{Sup. Tabular}} \\
LightGBM & 68.56 & 73.92 & 53.60 & 53.05 & 68.04 & 79.93 & 70.09 & 52.22 & 62.54 & 63.39 & 63.43 & 55.21 & 63.67 & 21.42 \\
DS+LightGBM & 69.80 & 82.40 & 64.27 & 64.46 & 75.42 & 84.23 & 72.00 & 67.60 & 81.80 & 90.30 & 86.20 & 69.00 & 75.62 & 11.08 \\
DS+AutoGluon & 70.01 & 82.53 & 65.66 & 65.87 & 74.85 & 82.31 & 70.51 & 68.12 & 80.45 & 89.94 & 85.38 & 68.71 & 75.36 & 11.83 \\
DS+TabPFN-2.5 & 71.16 & 77.70 & 64.21 & 64.53 & 76.43 & 85.70 & 70.75 & 66.28 & 79.83 & 89.07 & 85.08 & 68.31 & 74.92 & 12.92 \\
\midrule
\multicolumn{15}{@{}l}{\textit{Sup. Relational}} \\
GraphSAGE & 72.62 & 75.54 & 65.90 & 66.20 & 76.89 & 81.62 & 68.60 & \cellcolor{top3}70.42 & \cellcolor{top2}82.81 & \cellcolor{top2}90.59 & \cellcolor{top2}88.86 & 69.88 & 75.83 & 8.58 \\
HGT & 70.77 & 70.75 & 63.76 & 64.32 & 64.96 & 82.47 & 58.37 & 66.43 & 77.97 & 88.47 & 86.08 & 66.95 & 71.78 & 16.33 \\
HGT$_{\mathrm{PE}}$ & 71.17 & 76.27 & 64.57 & 64.95 & 65.36 & 81.61 & 59.21 & 66.19 & 78.03 & 88.17 & 85.66 & 65.69 & 72.24 & 15.50 \\
RelGNN & 75.29 & 85.69 & 68.23 & 66.18 & 79.61 & 86.18 & 71.24 & \cellcolor{top1}\textbf{70.99} & \cellcolor{top3}82.64 & \cellcolor{top1}\textbf{90.75} & \cellcolor{top1}\textbf{88.98} & \cellcolor{top2}70.93 & 78.06 & 5.17 \\
RelGT & 75.87 & 83.52 & \cellcolor{top3}68.30 & \cellcolor{top3}66.78 & 76.09 & 81.57 & 68.61 & 70.39 & 82.55 & \cellcolor{top3}90.53 & 86.32 & 69.27 & 76.65 & 8.00 \\
\midrule
\multicolumn{15}{@{}l}{\textit{Foundation Models}} \\
LLM$_1$ & 75.80 & 91.40 & 59.80 & 62.70 & 71.40 & 69.30 & 57.40 & 58.10 & 62.10 & 78.00 & 80.00 & 59.80 & 68.82 & 18.00 \\
LLM$_2$ & 80.03 & 87.11 & 53.36 & 54.07 & 70.11 & 68.65 & 59.17 & 62.55 & 73.41 & 81.23 & 79.99 & 63.81 & 69.46 & 17.58 \\
Rel-Zero & 71.84 & 70.64 & 62.28 & 56.17 & 68.12 & 61.32 & 59.02 & 60.07 & 64.10 & 69.46 & 62.12 & 55.95 & 63.42 & 20.67 \\
TabPFN-2.5 & 61.21 & 77.35 & 54.44 & 51.45 & 69.10 & 70.91 & 71.74 & 57.14 & 59.30 & 59.66 & 52.26 & 54.99 & 61.63 & 21.17 \\
Griffin & 57.70 & 82.50 & 45.90 & 60.70 & 71.88 & 83.27 & 51.00 & 62.30 & 69.00 & 77.50 & 73.50 & 60.20 & 66.29 & 19.75 \\
RT$_{\mathrm{zero}}$ & 81.20 & 89.30 & 59.50 & 61.80 & 73.22 & 77.47 & 51.80 & 64.00 & 70.90 & 75.70 & 80.10 & 62.80 & 70.65 & 17.00 \\
GNN+TabPFN & 64.26 & 74.67 & 60.68 & 65.35 & 73.58 & 84.13 & 56.68 & 64.30 & 77.89 & 87.48 & 85.34 & 64.34 & 71.56 & 16.67 \\
RDBLearn & 70.87 & 79.69 & \cellcolor{top1}\textbf{69.04} & 65.49 & 75.04 & 82.52 & 71.58 & 67.57 & 82.07 & 89.39 & 85.26 & 68.05 & 75.55 & 11.42 \\
KumoRFM & 82.41 & 91.07 & 64.85 & 64.11 & 76.08 & \cellcolor{top2}89.20 & 70.79 & 67.29 & 79.93 & 87.09 & 80.00 & 67.71 & 76.71 & 11.33 \\
KumoRFM-2 & \cellcolor{top2}84.59 & \cellcolor{top3}92.18 & 67.42 & \cellcolor{top1}\textbf{69.41} & \cellcolor{top3}81.66 & \cellcolor{top1}\textbf{90.83} & 72.03 & 69.10 & 82.17 & 89.40 & 87.15 & 69.27 & \cellcolor{top3}79.60 & 5.17 \\
\midrule
\multicolumn{15}{@{}l}{\textit{LLM as FE.}} \\
FeatLLM & 70.43 & 52.72 & 48.69 & 51.15 & 46.67 & 47.88 & 57.35 & - & 55.11 & 42.56 & - & - & - & - \\
ReFuGe & 76.26 & 83.04 & 63.20 & 66.43 & 77.14 & 83.40 & 72.22 & - & 67.00 & 87.66 & 83.80 & 68.10 & - & - \\
\textsc{RelAgent} & 78.34 & 85.23 & \cellcolor{top2}68.36 & \cellcolor{top2}67.79 & 78.20 & 87.25 & 71.86 & \cellcolor{top2}70.78 & \cellcolor{top1}\textbf{82.84} & 90.41 & 88.42 & \cellcolor{top1}\textbf{71.07} & 78.38 & \cellcolor{top3}4.50 \\
\midrule
\openrfm{} (R1) & \cellcolor{top3}84.11 & \cellcolor{top2}92.55 & 64.93 & 65.91 & 78.48 & 87.36 & \cellcolor{top3}72.75 & 67.67 & 81.32 & 89.99 & 88.12 & \cellcolor{top3}70.17 & 78.61 & 6.08 \\
\hspace{0.4em}+ harness (R2) & \cellcolor{top1}\textbf{84.94} & \cellcolor{top1}\textbf{93.22} & 67.83 & 66.23 & \cellcolor{top2}81.93 & 88.89 & \cellcolor{top2}74.34 & 69.66 & 82.24 & 89.70 & \cellcolor{top3}88.54 & 70.08 & \cellcolor{top2}79.80 & \cellcolor{top2}4.00 \\
\hspace{0.4em}+ head (R3) & \cellcolor{top1}\textbf{84.94} & \cellcolor{top1}\textbf{93.22} & 67.84 & 66.22 & \cellcolor{top1}\textbf{82.91} & \cellcolor{top3}89.06 & \cellcolor{top1}\textbf{74.37} & 69.75 & 82.24 & 89.71 & \cellcolor{top3}88.54 & 70.08 & \cellcolor{top1}\textbf{79.91} & \cellcolor{top1}\textbf{3.58} \\
\bottomrule
\end{tabular}}
\end{table}

%% file: tables/main_relbench_v1_reg.tex
\begin{table}[t]
\centering
\caption{Test MAE ($\downarrow$) on the RelBench-V1 regression tasks. $\mu_n$ is the MAE normalised by LightGBM and averaged over tasks.}
\label{tab:main_relbench_v1_reg}
\small
\setlength{\tabcolsep}{3pt}
\resizebox{\linewidth}{!}{%
\begin{tabular}{@{}lrrrrrrrrrrr@{}}
\toprule
Method & \multicolumn{1}{c}{f1} & \multicolumn{1}{c}{avito} & \multicolumn{1}{c}{event} & \multicolumn{2}{c}{trial} & \multicolumn{2}{c}{amazon} & \multicolumn{1}{c}{stack} & \multicolumn{1}{c}{hm} & $\mu_n$ & Rank \\
\cmidrule(lr){2-2}\cmidrule(lr){3-3}\cmidrule(lr){4-4}\cmidrule(lr){5-6}\cmidrule(lr){7-8}\cmidrule(lr){9-9}\cmidrule(lr){10-10}
 & pos & ctr & attend & adverse & success & user & item & votes & sales & $\downarrow$ & $\downarrow$ \\
\midrule
Global Median & 4.399 & 0.043 & 0.264 & 57.533 & 0.462 & 16.783 & 64.234 & \cellcolor{top3}0.068 & 0.076 & 1.062 & 15.67 \\
Entity Median & 8.519 & 0.046 & 0.269 & 57.930 & 0.441 & 17.423 & 66.436 & 0.069 & 0.078 & 1.190 & 18.00 \\
\midrule
\multicolumn{12}{@{}l}{\textit{Sup. Tabular}} \\
LightGBM & 4.170 & 0.041 & 0.264 & 44.011 & 0.425 & 16.783 & 60.569 & \cellcolor{top3}0.068 & 0.076 & 1.000 & 13.00 \\
DS+LightGBM & 3.963 & 0.044 & 0.284 & 40.581 & 0.407 & 13.928 & \cellcolor{top1}\textbf{41.122} & \cellcolor{top2}0.065 & \cellcolor{top3}0.036 & 0.880 & 7.33 \\
DS+AutoGluon & 4.251 & 0.045 & 0.256 & 44.706 & 0.430 & 14.399 & 45.390 & \cellcolor{top3}0.068 & 0.043 & 0.921 & 11.00 \\
DS+TabPFN-2.5 & 4.373 & 0.044 & 0.318 & 47.168 & 0.432 & 15.631 & 47.908 & 0.080 & 0.059 & 1.010 & 14.00 \\
\midrule
\multicolumn{12}{@{}l}{\textit{Sup. Relational}} \\
GraphSAGE & 4.022 & 0.041 & 0.258 & 44.473 & 0.400 & 14.313 & 50.053 & \cellcolor{top2}0.065 & 0.056 & 0.918 & 9.78 \\
HGT & 4.226 & 0.046 & 0.264 & 45.169 & 0.443 & 15.412 & 55.868 & \cellcolor{top3}0.068 & 0.064 & 0.987 & 14.00 \\
HGT$_{\mathrm{PE}}$ & 4.392 & 0.048 & 0.261 & 42.648 & 0.440 & 15.864 & 55.849 & \cellcolor{top3}0.068 & 0.064 & 0.992 & 13.44 \\
RelGNN & 3.798 & 0.037 & 0.238 & 44.461 & \cellcolor{top1}\textbf{0.301} & 14.230 & 48.767 & \cellcolor{top2}0.065 & 0.054 & 0.861 & 7.22 \\
RelGT & 3.917 & 0.035 & 0.250 & 43.992 & \cellcolor{top2}0.326 & 14.267 & 48.922 & \cellcolor{top2}0.065 & 0.054 & 0.870 & 7.56 \\
\midrule
\multicolumn{12}{@{}l}{\textit{Foundation Models}} \\
TabPFN-2.5 & 4.446 & 0.044 & 0.527 & 55.187 & 0.469 & 17.513 & 60.996 & 0.104 & 0.096 & 1.260 & 18.67 \\
Griffin & 4.460 & 0.050 & 0.461 & 78.232 & 0.463 & 35.590 & 53.214 & 0.092 & 0.151 & 1.471 & 19.44 \\
RT$_{\mathrm{zero}}$ & 2.901 & 0.058 & 0.379 & 73.999 & 0.455 & 18.802 & 57.996 & 0.110 & 0.089 & 1.240 & 17.89 \\
RDBLearn & 3.834 & \cellcolor{top3}0.034 & 0.237 & 43.913 & 0.424 & 14.540 & 48.559 & \cellcolor{top3}0.068 & 0.064 & 0.906 & 8.89 \\
KumoRFM & 2.747 & 0.035 & 0.264 & 58.231 & 0.417 & 16.161 & 55.254 & \cellcolor{top2}0.065 & 0.040 & 0.908 & 10.22 \\
KumoRFM-2 & 2.854 & \cellcolor{top3}0.034 & \cellcolor{top2}0.235 & 43.293 & 0.433 & \cellcolor{top3}13.921 & 46.992 & \cellcolor{top1}\textbf{0.064} & \cellcolor{top1}\textbf{0.034} & 0.822 & 5.11 \\
\textsc{RelAgent} & 4.019 & \cellcolor{top2}0.033 & 0.241 & \cellcolor{top3}37.194 & 0.386 & 13.949 & \cellcolor{top2}41.765 & \cellcolor{top1}\textbf{0.064} & \cellcolor{top2}0.035 & 0.817 & \cellcolor{top3}4.44 \\
\midrule
\openrfm{} (R1) & \cellcolor{top3}2.608 & \cellcolor{top1}\textbf{0.032} & \cellcolor{top3}0.236 & 37.498 & 0.392 & 14.149 & 45.638 & \cellcolor{top2}0.065 & 0.037 & \cellcolor{top3}0.789 & 4.56 \\
\hspace{0.4em}+ harness (R2) & \cellcolor{top1}\textbf{2.598} & \cellcolor{top1}\textbf{0.032} & \cellcolor{top1}\textbf{0.232} & \cellcolor{top2}35.902 & 0.364 & \cellcolor{top2}13.904 & 42.232 & \cellcolor{top1}\textbf{0.064} & \cellcolor{top3}0.036 & \cellcolor{top2}0.767 & \cellcolor{top2}2.11 \\
\hspace{0.4em}+ head (R3) & \cellcolor{top2}2.599 & \cellcolor{top1}\textbf{0.032} & \cellcolor{top1}\textbf{0.232} & \cellcolor{top1}\textbf{35.813} & \cellcolor{top3}0.363 & \cellcolor{top1}\textbf{13.883} & \cellcolor{top3}42.210 & \cellcolor{top1}\textbf{0.064} & \cellcolor{top3}0.036 & \cellcolor{top1}\textbf{0.766} & \cellcolor{top1}\textbf{1.78} \\
\bottomrule
\end{tabular}}
\end{table}

%% file: tables/main_relbench_v2.tex
\begin{table}[t]
\centering
\begin{minipage}[t]{0.50\linewidth}
\vspace{0pt}
\input{tables/main_4dbinfer_cls}

\end{minipage}\hfill
\begin{minipage}[t]{0.47\linewidth}
\vspace{0pt}
\centering
\caption{Test AUROC ($\uparrow$) on RelBench-V2.}
\label{tab:main_relbench_v2_cls}
\footnotesize
\setlength{\tabcolsep}{1.1pt}
{%
\begin{tabular}{@{}lrrrrrr@{}}
\toprule
Method & \multicolumn{3}{c}{ratebeer} & \multicolumn{1}{c}{arxiv} & Avg & Rank \\
\cmidrule(lr){2-4}\cmidrule(lr){5-5}
 & beer & user & dormant & citation & $\uparrow$ & $\downarrow$ \\
\midrule
LightGBM & 76.21 & 83.92 & 75.79 & 71.21 & 76.78 & 7.75 \\
GraphSAGE & 78.67 & 94.27 & 80.51 & \cellcolor{top2}82.50 & 83.99 & 5.00 \\
KumoRFM & 75.06 & 91.38 & 77.10 & 80.62 & 81.04 & 7.25 \\
KumoRFM-2 & 83.84 & \cellcolor{top2}97.43 & 80.65 & 81.71 & 85.91 & 4.25 \\
\textsc{RelAgent} & \cellcolor{top3}84.70 & \cellcolor{top1}\textbf{98.63} & \cellcolor{top1}\textbf{83.33} & \cellcolor{top1}\textbf{82.62} & \cellcolor{top1}\textbf{87.32} & \cellcolor{top1}\textbf{1.50} \\
\midrule
\openrfm{} (R1) & \cellcolor{top2}84.73 & 95.66 & \cellcolor{top3}81.96 & 81.30 & 85.91 & 4.25 \\
\hspace{0.4em}+ harness (R2) & 84.66 & \cellcolor{top3}97.03 & \cellcolor{top2}82.52 & 82.23 & \cellcolor{top3}86.61 & \cellcolor{top3}3.25 \\
\hspace{0.4em}+ head (R3) & \cellcolor{top1}\textbf{85.23} & \cellcolor{top3}97.03 & \cellcolor{top2}82.52 & \cellcolor{top3}82.26 & \cellcolor{top2}86.76 & \cellcolor{top2}2.25 \\
\bottomrule
\end{tabular}}

\vspace{10pt}
\centering
\caption{Test MAE ($\downarrow$) on RelBench-V2. $\mu_n$ as in Table~\ref{tab:main_relbench_v1_reg}.}
\label{tab:main_relbench_v2_reg}
\footnotesize
\setlength{\tabcolsep}{1.7pt}
{%
\begin{tabular}{@{}lrrrr@{}}
\toprule
Method & \multicolumn{1}{c}{ratebeer} & \multicolumn{1}{c}{arxiv} & $\mu_n$ & Rank \\
\cmidrule(lr){2-2}\cmidrule(lr){3-3}
 & user-count & publication & $\downarrow$ & $\downarrow$ \\
\midrule
Global Median & 15.124 & 0.577 & 0.872 & 8.50 \\
Entity Median & 13.079 & 0.874 & 1.079 & 9.00 \\
LightGBM & 20.350 & 0.577 & 1.000 & 9.00 \\
GraphSAGE & 7.374 & 0.513 & 0.626 & 6.00 \\
KumoRFM & 11.063 & 0.518 & 0.721 & 7.00 \\
KumoRFM-2 & 7.298 & \cellcolor{top3}0.487 & 0.601 & 5.00 \\
\textsc{RelAgent} & \cellcolor{top1}\textbf{6.021} & \cellcolor{top2}0.462 & \cellcolor{top1}\textbf{0.548} & \cellcolor{top3}2.50 \\
\midrule
\openrfm{} (R1) & 7.103 & \cellcolor{top1}\textbf{0.459} & 0.572 & \cellcolor{top3}2.50 \\
\hspace{0.4em}+ harness (R2) & \cellcolor{top3}6.484 & \cellcolor{top1}\textbf{0.459} & \cellcolor{top3}0.557 & \cellcolor{top2}2.00 \\
\hspace{0.4em}+ head (R3) & \cellcolor{top2}6.410 & \cellcolor{top1}\textbf{0.459} & \cellcolor{top2}0.556 & \cellcolor{top1}\textbf{1.50} \\
\bottomrule
\end{tabular}}
\end{minipage}
\end{table}

%% file: tables/main_4dbinfer_cls.tex
\centering
\captionof{table}{Test AUROC ($\uparrow$) on 4DBInfer classification.}
\label{tab:main_4dbinfer_cls}
\footnotesize
\setlength{\tabcolsep}{1.2pt}
\renewcommand{\arraystretch}{0.95}
{%
\begin{tabular}{@{}lrrrrrrr@{}}
\toprule
Method & \multicolumn{1}{c}{AB} & \multicolumn{1}{c}{OB} & \multicolumn{1}{c}{RR} & \multicolumn{2}{c}{SE} & Avg & Rank \\
\cmidrule(lr){2-2}\cmidrule(lr){3-3}\cmidrule(lr){4-4}\cmidrule(lr){5-6}
 & churn & ctr & cvr & vote & churn & $\uparrow$ & $\downarrow$ \\
\midrule
\multicolumn{8}{@{}l}{\textit{Sup. Tab.}} \\
XGBoost & 50.00 & 50.00 & 50.00 & 49.68 & 50.84 & 50.10 & 14.40 \\
AutoGluon & 50.00 & 49.69 & 50.96 & 50.81 & 50.00 & 50.29 & 14.40 \\
DFS+XGBoost & 69.22 & 54.21 & 79.06 & 86.75 & 82.51 & 74.35 & 13.00 \\
DFS+AutoGluon & 72.91 & 54.94 & 80.08 & 88.49 & 83.96 & 76.08 & 11.20 \\
\midrule
\multicolumn{8}{@{}l}{\textit{Sup. Rel.}} \\
GAT & 76.22 & 61.46 & 82.84 & 88.53 & 86.45 & 79.10 & 9.00 \\
HGT & 77.30 & \cellcolor{top3}62.60 & 84.95 & 88.17 & 86.70 & 79.94 & 6.40 \\
PNA & 76.45 & 62.49 & 83.67 & \cellcolor{top2}88.96 & 86.64 & 79.64 & 6.60 \\
GraphSAGE & 75.71 & 62.39 & 84.70 & 88.61 & 85.58 & 79.40 & 7.60 \\
\midrule
\multicolumn{8}{@{}l}{\textit{Found.}} \\
RDBLearn & \cellcolor{top3}77.41 & 54.47 & 84.69 & 88.45 & 87.96 & 78.60 & 7.40 \\
KumoRFM & 75.59 & 61.11 & 84.36 & \cellcolor{top3}88.92 & 87.87 & 79.57 & 7.80 \\
KumoRFM-2 & 77.06 & 61.97 & 84.36 & 88.50 & 87.92 & 79.96 & 7.20 \\
\textsc{RelAgent} & \cellcolor{top1}\textbf{79.44} & \cellcolor{top2}62.63 & \cellcolor{top3}86.35 & \cellcolor{top1}\textbf{89.45} & \cellcolor{top3}89.03 & \cellcolor{top3}81.38 & \cellcolor{top2}2.20 \\
\midrule
\openrfm{} (R1) & 76.13 & 62.28 & 84.38 & 86.93 & 88.86 & 79.72 & 7.80 \\
\hspace{0.4em}+ harness (R2) & \cellcolor{top2}79.28 & \cellcolor{top1}\textbf{63.06} & \cellcolor{top2}86.43 & 88.89 & \cellcolor{top2}89.81 & \cellcolor{top2}81.49 & \cellcolor{top3}2.40 \\
\hspace{0.4em}+ head (R3) & \cellcolor{top2}79.28 & \cellcolor{top1}\textbf{63.06} & \cellcolor{top1}\textbf{86.44} & 88.91 & \cellcolor{top1}\textbf{89.82} & \cellcolor{top1}\textbf{81.50} & \cellcolor{top1}\textbf{1.80} \\
\bottomrule
\end{tabular}}

%% file: tables/ablation_all.tex
\caption{Ablation study, mean rank. (a) and (b) use the 9-task subset and (c) all 32 tasks, where a task that an axis does
not apply to keeps its R2 rank. In (c), R2 is re-scored with random support order, so its rank is 2.469 instead of 2.406. Each indented row changes the named components of the row it is nested under.}
\label{tab:ablation}
\setlength{\tabcolsep}{0pt}
\begin{tabular*}{\linewidth}{@{\extracolsep{\fill}}l@{\hspace{4pt}}r@{\hspace{14pt}}l@{\hspace{4pt}}r@{\hspace{14pt}}l@{\hspace{4pt}}r@{}}
\toprule
\textbf{(a) Pre-training} & Rank & \textbf{(b) Architecture} & Rank & \textbf{(c) Harness} & Rank \\
\cmidrule(r{8pt}){1-2}\cmidrule(lr{8pt}){3-4}\cmidrule(l){5-6}
R1 & 3.111 & R1 + Kumo-Tabular small & 3.111 & R2, random support order & 2.469 \\
\quad $-$ memory module & 5.667 & \quad $-$ memory module & 5.667 & \quad\emph{Schema program} & \\
\qquad shorter pre-training & 5.778 & \quad $-$ both target-cell states & 4.111 & \qquad $-$ table rewrite & 4.938 \\
\qquad\quad $-$ autocomplete family & 5.889 & \quad $-$ middle-block state & 5.667 & \qquad $-$ target transform & 2.656 \\
\qquad\quad $-$ future-activity families & 6.000 & \quad $-$ final-block state & 3.444 & \quad $-$ rewrite and selection tuning & 4.969 \\
\cmidrule(r{8pt}){1-2}\cmidrule(lr{8pt}){3-4}\cmidrule(l){5-6}
\emph{Other pre-training priors} & & \emph{R1 with other TFM backbones} & & \emph{Support set} & \\
PluRel, Kumo-Tabular readout & 6.444 & R1 + Kumo-Tabular large & 3.667 & \quad $-$ selection tuning & 3.031 \\
PluRel, original RT head & 12.556 & R1 + TabICL v2 & 5.778 & & \\
& & R1 + TabPFN-3.5 & 4.667 & & \\
\bottomrule
\end{tabular*}

%% file: figures/scaling.tex
\centering
\includegraphics[width=\linewidth]{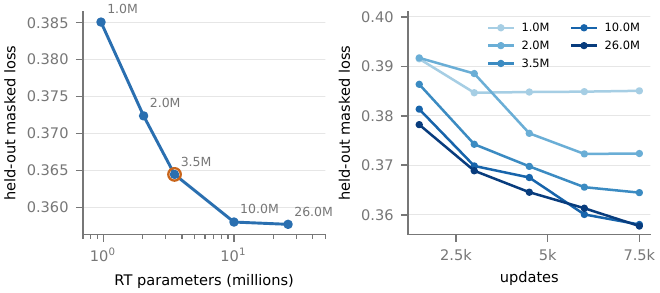}
\captionof{figure}{Model size. Left, final masked-cell loss of five adapters on 612 held-out synthetic worlds
after 7,500 updates. Right, the loss over training. The circle marks the size we use.}
\label{fig:scaling}

%% file: sections/related.tex
\section{Related work}
\label{sec:related}

\textbf{Relational deep learning.} Relational deep learning~\citep{fey2024rdl} turns a database into a heterogeneous
graph and trains a graph neural network for each task, with RelBench~\citep{robinson2024relbench,gu2026relbenchv2} and
4DBInfer~\citep{wang2024fourdbinfer} as standard benchmarks. Later work adds composite message
passing~\citep{chen2025relgnn}, architectures for recommendation~\citep{yuan2024contextgnn}, graph
transformers~\citep{dwivedi2025rgt} and LLM-based predictors~\citep{wu2025llmrelational,wydmuch2024llmrdb}.
AutoG~\citep{chen2025autog} and RDB2G-Bench~\citep{choi2025rdb2gbench} study how a database should become a graph, which
is close to the schema rewrites of our harness. All of these models are trained for each task.

\textbf{Relational foundation models.} Relational foundation models avoid training for each task and predict from a
labelled support set in context. KumoRFM~\citep{fey2025kumorfm,hudovernik2026kumorfm} and Griffin~\citep{wang2025griffin}
pre-train graph models for in-context prediction on relational data. RT~\citep{ranjan2025relational} treats each cell of
a sampled context as a token, PluRel~\citep{kothapalli2026plurel} pre-trains it on synthetic databases, and
RT-J~\citep{ranjan2026rtj} scales its pre-training to longer contexts. Another line flattens the database for a tabular
model. RDBLearn~\citep{zhang2026rdblearn} feeds relational aggregates to a TFM, RDB-PFN~\citep{wang2026relational}
pre-trains a prior-fitted model on synthetic databases, and TabPFN-Rel~\citep{hayler2026relarena} builds a harness around
TabPFN-3 and compares these baselines on a common footing.

\textbf{Prior-fitted models for tables and graphs.} Prior-data fitted networks~\citep{muller2022pfn} learn Bayesian
prediction from synthetic data, and TabPFN, TabICL and Kumo-Tabular~\citep{hollmann2025tabpfn,qu2025tabicl,nvidia2026kumotabular}
bring this recipe to single tables. G2T-FM~\citep{eremeev2025g2tfm}, GraphPFN~\citep{eremeev2025graphpfn} and NodePFN~\citep{choi2026nodepfn} carry it
to node classification on one graph, where the support set is a labelled subset of its nodes. Training-free
methods~\citep{xu2026noneedtotrain} reuse a tabular model on fixed relational aggregates. Our prior
differs in that it generates event processes in time, and our adapter learns which aggregates to compute.

\textbf{Graph foundation models.} ULTRA~\citep{galkin2024ultra} transfers across knowledge graphs, and other work builds
one backbone that respects the symmetries of any graph~\citep{mao2024gfm,wang2025gfmsurvey}. Such transfer works well on
knowledge-graph reasoning but remains hard for general node and link tasks, and different task levels may require
symmetries that one backbone cannot combine~\citep{bevilacqua2025holographic}. \openrfm{} fixes the task as prediction
for an entity at a cutoff, so it needs no such universal backbone.

\textbf{LLMs for features and context.} CAAFE~\citep{hollmann2023caafe} and FeatLLM~\citep{han2024featllm} use LLMs to
write features for a single table, and RelAgent~\citep{huang2026relagent} writes SQL features over a database. For LLM
agents, recent work shows that the content of the context matters as much as the model~\citep{liu2024lost,zhang2025ace}.
Our harness applies this view to an RFM. The LLM decides what enters the context of a fixed model.

%% file: sections/conclusion.tex
\section{Conclusion}
\label{sec:conclusion}

We present \openrfm{}, which pairs a relational adapter pre-trained only on synthetic databases with a frozen tabular
foundation model and an LLM harness that decides what enters the context. On the 32 benchmark tasks it achieves the best
mean rank without training on the evaluation databases.

\textbf{Limitations and open questions.} Our adapter has 3.5M parameters and is pre-trained with limited computation resources, so how far
larger adapters and longer pre-training would go remains untested. The LLM in the harness is used without adaptation, and
post-training it on the feedback of the harness could make the search cheaper and more reliable. Inference already
uses CUDA graphs and exact sparse kernels, and new architectures could make it faster still. Finally, it remains open whether relational knowledge is best added through an adapter in front of a frozen
tabular model, as in \openrfm{}, or built directly into the tabular backbone during its pre-training. We see these questions as directions for future
research.

%% file: tables/latency.tex
\begin{table}[!htb]
\centering
\caption{Inference time in seconds on \textsc{rel-f1} tasks, one B200 GPU.}
\label{tab:latency}
\small
\begin{tabular}{@{}lrrr@{}}
\toprule
Method & driver-position & driver-dnf & driver-top3 \\
\midrule
RT-J & 17.56 & 16.35 & 16.79 \\
\openrfm{} default & 31.91 & 34.49 & 9.83 \\
\openrfm{} with CUDA graphs & 14.43 & 15.92 & 4.68 \\
\bottomrule
\end{tabular}
\end{table}

%% file: tables/label_families.tex
\begin{table}[!htb]
\centering
\caption{Label families of adapter pre-training and world families of memory-module training, grouped by the evidence
the label depends on. Shares are fractions of pre-training episodes and of the 64 worlds of memory-module training.}
\label{tab:targets_full}
\small
\renewcommand{\arraystretch}{1.2}
\begin{tabularx}{\linewidth}{@{}>{\raggedright\arraybackslash}p{3.7cm}>{\raggedright\arraybackslash}X>{\raggedright\arraybackslash}p{5.6cm}@{}}
\toprule
Evidence & Adapter pre-training & Memory-module training \\
\midrule
Own history in a window & future window 18.5\%, window activity 5.0\% & memory signal 1.6\%, exact read 1.6\%, censoring 39.1\% \\
Own and peer aggregates & derived column 16.0\%, sets and lookups 8.8\%, event programs 1.9\% & peer excitation 1.6\% \\
Past labels and drift & persistent history 13.0\% & regime change 1.6\%, heavy tail 1.6\% \\
Row and parent values & autocomplete 11.2\%, row and support functions 10.6\%, fact-row autocomplete 10.0\% & near cells 45.3\% \\
History with row values & & joint signal, two families 3.1\% \\
Items offered together & group choice 5.0\% & \\
Other & & replay 3.1\%, silence 1.6\% \\
\bottomrule
\end{tabularx}
\end{table}